%% file: main.tex
\documentclass{article}

\PassOptionsToPackage{numbers, compress}{natbib}
\usepackage{natbib}

\usepackage{algorithm}
\usepackage{algorithmic}
\usepackage{tikz}
\usepackage{caption}
\usepackage{subcaption} 
\usepackage[utf8]{inputenc} 
\usepackage[T1]{fontenc}    
\usepackage{hyperref}       
\usepackage{url}            
\usepackage{booktabs}       
\usepackage{amsfonts, amsmath, amsthm}       
\usepackage{amssymb}  
\usepackage{nicefrac}       
\usepackage{microtype}      
\usepackage{xcolor}         
\usepackage{multirow} 
\usepackage{booktabs} 
\usepackage{accents}
\usepackage{comment}
\usepackage{wrapfig}
\usepackage{graphicx}
\usepackage{enumitem}
\usepackage{cleveref}

\usepackage[margin=1.1in]{geometry}

\setlist[enumerate]{leftmargin=15pt, label=\roman*)}
\setlist[itemize]{leftmargin=10pt}

\author{%
    Anna van Elst \\
    LTCI, T\'el\'ecom Paris, Institut Polytechnique de Paris \\
  \texttt{anna.vanelst@telecom-paris.fr}
    \and
  Debarghya Ghoshdastidar\\
  School of Computation Information and Technology, \\
Technical University of Munich \\
  \texttt{ghoshdas@cit.tum.de}
}
\date{}

\newtheorem{lemma}{Lemma}
\newtheorem{definition}{Definition}
\newtheorem{theorem}{Theorem}
\newtheorem{corollary}{Corollary}
\newtheorem{remark}{Remark}
\newtheorem{proposition}{Proposition}

\newtheorem*{theorem*}{Theorem}

\title{Tight PAC-Bayesian Risk Certificates for Contrastive Learning}

\begin{document}

\maketitle

\begin{abstract}
\input{files/0_abstract}

\end{abstract}

\input{files/1_introduction}
\input{files/2_preliminaries}
\input{files/3_theory}

\input{files/4_experiments}
\input{files/5_discussion}

\bibliography{main}

\appendix

\input{files/6_app_bounds}
\input{files/7_app_exp}
\input{files/8_app_proof}
\input{files/9_app_extensions}

\end{document}

%% file: files/0_abstract.tex
Contrastive representation learning is a modern paradigm for learning representations of unlabeled data via augmentations---precisely, contrastive models learn to embed semantically similar pairs of samples (positive pairs) closer than independently drawn samples (negative samples). In spite of its empirical success and widespread use in foundation models, statistical theory for contrastive learning remains less explored.
Recent works have developed generalization error bounds for contrastive losses, but the resulting risk certificates are either vacuous (certificates based on Rademacher complexity or $f$-divergence) or require strong assumptions about samples that are unreasonable in practice.
The present paper develops non-vacuous PAC-Bayesian risk certificates for contrastive representation learning, considering the practical considerations of the popular SimCLR framework. 
Notably, we take into account that SimCLR reuses positive pairs of augmented data as negative samples for other data, thereby inducing strong dependence and making classical PAC or PAC-Bayesian bounds inapplicable. 
We further refine existing bounds on the downstream classification loss by incorporating SimCLR-specific factors, including data augmentation and temperature scaling, and derive risk certificates for the contrastive zero-one risk. 
The resulting bounds for contrastive loss and downstream prediction are much tighter than those of previous risk certificates, as demonstrated by experiments on CIFAR-10. 

%% file: files/1_introduction.tex
\section{Introduction}
A key driving force behind the rapid advances in foundation models is the availability and exploitation of massive amounts of unlabeled data. Broadly, one learns meaningful representations from unlabeled data, reducing the demand for labeled samples when training (downstream) predictive models. 
In recent years, there has been a strong focus on self-supervised approaches to representation learning, which learn neural network-based embedding maps from carefully constructed augmentations of unlabeled data, such as image cropping, rotations, color distortion, Gaussian blur, etc. \cite{AssranDMBVRLB23,chen2020simple,ChenH21,GrillSATRBDPGAP20}.

Contrastive representation learning is a popular form of self-supervised learning where one aims to extract meaningful features from unlabeled data by distinguishing between similar samples, obtained via augmentations, and dissimilar samples \cite{wu2018unsupervised, he2020momentum,le2020contrastive}. It has shown strong empirical performance across a wide range of of applications \cite{hu2024comprehensive}, such as image recognition \cite{chen2020simple, zhang2022contrastive}, natural language processing \cite{gao2021simcse}, audio processing \cite{elizalde2023clap}, and graph processing \cite{yu2023xsimgcl, qiu2020gcc}. The core idea is to learn an embedding space where similar (positive) samples are pulled closer together, while dissimilar (negative) samples are pushed apart. This is typically accomplished using a contrastive loss function applied to features extracted by an encoder. A standard contrastive learning pipeline begins with data augmentation, where two augmented views of the same sample are generated to form a positive pair, while other samples in the batch serve as negatives. These augmented views are passed through an encoder network, optionally followed by a projection head. The network is then trained using a contrastive loss computed on these feature representations. Central to many contrastive frameworks is the InfoNCE loss \cite{sohn2016improved, gutmann2010noise, oord2018representation}, which uses a softmax function over pairwise feature similarities, often incorporating feature normalization and temperature scaling. This loss is employed in popular frameworks such as SimCLR \cite{chen2020simple}, MoCo \cite{he2020momentum}, and CLIP \cite{radford2021learning}, where negative samples play a critical role. In contrast, other methods like Barlow Twins \cite{zbontar2021barlow} and VICReg \cite{bardes2021vicreg} eliminate the need for negative sampling by focusing on constraining the cross-correlation or cross-covariance matrix, while still ensuring alignment between positive pairs.

In this work, we focus on the SimCLR framework (\textit{i.e.,} \textit{simple framework for contrastive learning of representations}), which has gained widespread attention and demonstrated strong performance in a variety of downstream tasks\cite{chen2020simple}. Although SimCLR remains one of the most practically used contrastive models, theoretical analysis of SimCLR’s performance and generalization abilities is still limited \cite{bao2022surrogate, nozawa2020pac}. The study of generalization error in self-supervised models is mostly based on two distinct frameworks \cite{arora2019theoretical,haochen2021provable}, both introduced in the context of contrastive learning. The \emph{contrastive unsupervised representation learning} (CURL) framework, introduced by Arora et al. \cite{arora2019theoretical}, 
assumes access to tuples $z_1,\ldots,z_n$, where each \(z= (x, x^+, x^{-}_1, \ldots, x^{-}_k)\).
The underlying statistical model considers data from a mixture of $k$ (class) distributions, and within each $z$, it is assumed that \(x\) and \(x^+\) are independent samples from the same class, while \(x^{-}_1, \ldots, x^{-}_k\) are \(k\) i.i.d. samples from the mixture model, independent of \(x,x^+\). Arora et al. \cite{arora2019theoretical} derive Rademacher complexity-based generalization error (risk) bounds for representation learning using bound contrastive losses. Furthermore, the above statistical model for augmented data $z$ allows one to extend the bounds to derive a risk certificate for downstream classification using the mean classifier \cite{arora2019theoretical}.
Subsequently, improved risk certificates for both contrastive learning and downstream classification were derived in \cite{LeiYYZ23} using empirical covering numbers, and also in \cite{nozawa2020pac} using a PAC-Bayesian analysis or in \cite{bao2022surrogate} --- the latter bounds being specifically obtained for the N-pair (or InfoNCE) loss used in SimCLR \cite{sohn2016improved,chen2020simple}.
The CURL framework has been used to derive risk bounds for non-contrastive models  \cite{CabannesKBLB23} and adversarial contrastive models \cite{ZouL23}, although both rely on Rademacher bounds. PAC-Bayes bounds have been instead stated in the  context of meta-learning \cite{FaridM21}.

The independence of samples assumed in the CURL framework is quite impractical. For instance, in the SimCLR framework, the positive samples $x,x^+$ are augmented views of the same data instances, for example, random rotation or cropping of the same image. In addition, practical implementations of SimCLR do not generate independent negative samples. The loss is computed over batches of positive pairs $(x_i, x_i^+)_{i=1,\ldots,m}$ and for each $x_i$ in a batch, all $x_j,x_j^+, j\neq i$ are used as negative samples, thereby inducing strong dependence across tuples $z_1,\ldots,z_n$ and making classical PAC-Bayes analysis inapplicable.
Nozawa et al. \cite{nozawa2020pac} derive $f$-divergence-based PAC-Bayes bounds to account for potential dependence across tuples $z_1,\ldots,z_n$; however, their resulting bounds are vacuous in practice.

The impractical assumptions of CURL were noted by HaoChen et al. \cite{haochen2021provable}, who in turn proposed risk certificates in terms of the augmentation graph: a graph over all samples, where each edge is weighted by the probability of obtaining both samples through independent random augmentations of the same data. This framework has proved to be quite useful for understanding inductive biases \cite{HaoChen023}, connections between contrastive learning and spectral methods \cite{haochen2021provable,TanZYY24}, the influence of data augmentation \cite{wang2022chaos,WeiSCM21}, bounds for unsupervised domain adaptation \cite{HaoChenWK022}, etc. However, the setting requires technical assumptions on the augmentation graph, whose practical applicability is not clear, and the works do not provide risk certificates that can be empirically verified (the variance-based bounds could be vacuous in practice). 

Other approaches to study generalization in self-supervised learning have been suggested, although their suitability for deriving practical risk certificates has not been sufficiently explored. For instance, information theoretic bounds have been derived in \cite{ToshK021,Tsai0SM21,ShwartzZivL24,ShwartzZivBKRL24}, generalization bounds under cluster assumptions on data have been suggested in \cite{0001YZJ23,ParulekarCSMS23}, and lower bounds on contrastive loss and cross-entropy risk have also been derived in \cite{GrafHNK21,NozawaS21}.

\begin{table}[t]
\centering
\begin{tabular}{@{\hspace{0.2cm}}c@{\hspace{0.2cm}}||@{\hspace{0.12cm}}c@{\hspace{0.12cm}}|@{\hspace{0.12cm}}c@{\hspace{0.12cm}}|@{\hspace{0.12cm}}c@{\hspace{0.12cm}}}
\bottomrule
 \textbf{Contrastive Loss} & Non-vacuous & Sample dependence & SimCLR \\ 
\hline
Arora et al. \cite{arora2019theoretical} & {$\times$} & $\times$ & $\times$ \\ 
Nozawa et al. \cite{nozawa2020pac} & \checkmark & \checkmark & $\times$ \\ 
Theorem \ref{thm:simclr-kl-pb}, \ref{thm:simclr-diarmid-pb} (ours) &  \checkmark & \checkmark & \checkmark \\ 
\hline
 \textbf{Downstream Risk} & Many $-$ve samples &  Data augmentation & Temperature scaling\\ 
\hline
Arora et al. \cite{arora2019theoretical} & $\times$ & $\times$ & $\times$ \\
Bao et al. \cite{bao2022surrogate} & \checkmark & $\times$ & $\times$ \\ 
Wang et al. \cite{wang2022chaos} & \checkmark & \checkmark & $\times$ \\ 
Theorem \ref{thm:downstream} (ours) & \checkmark & \checkmark & \checkmark\\ 
\toprule
\end{tabular}
\caption{Comparison of risk certificates for contrastive learning. A checkmark (\checkmark) indicates that the method satisfies the specified condition, and a cross ($\times$) denotes that it does not. The table is divided into two parts: the first part compares generalization bounds for contrastive learning, assessing whether they are non-vacuous, account for sample dependence, and are directly applicable to the SimCLR loss. Although Nozawa et al. do not provide a SimCLR-specific bound, their $f$-divergence bound can be extended to SimCLR, as detailed in the supplementary material. The second part compares bounds on downstream classification loss, focusing on their incorporation of a large number of negative samples, data augmentation, and temperature scaling in contrastive loss.}
\label{tab:related_works}
\end{table}

\subsection*{Motivation and contributions in this work}
The focus of the present work is to develop practical risk certificates for contrastive learning applicable to the widely used SimCLR framework. 
To this end, we consider an underlying data-generating model from which i.i.d. positive pairs can be sampled, which are used as negative samples for other samples---this framework is aligned with SimCLR.
Using a PAC-Bayesian approach, we derive two new risk certificates for the SimCLR loss. 
The key issue in using a PAC-Bayesian analysis is to account for the dependence across tuples (via negative samples).
Our main technical contribution is to show that current PAC-Bayes bounds can be improved when Hoeffding's and McDiarmid's inequalities are applied carefully \cite{hoeffding1994probability, mcdiarmid1989method}. 
We further build on recent advances in risk certificates for neural networks \cite{perez2021tighter, perez2021learning}, and, hence, arrive at certificates that are non-vacuous and significantly tighter than previous ones \cite{nozawa2020pac}. 
Additionally, most existing works do not consider practical settings that are known to enhance performance in the SimCLR framework, such as temperature scaling and large batch sizes \cite{bao2022surrogate, wang2022chaos}. 
We refine existing bounds on contrastive loss and downstream classification loss \cite{wang2022chaos, bao2022surrogate} by incorporating SimCLR-specific factors, including data augmentation and temperature scaling. We also extend our analysis to the contrastive zero-one risk \cite{nozawa2020pac} and derive corresponding risk certificates. A comparison of our risk certificates with existing bounds is summarized in Table \ref{tab:related_works}.

The paper is organized as follows. We provide background on the SimCLR framework and PAC-Bayes theory in \cref{sec:background}. We then present our main contribution in \cref{sec:contrib}. In \cref{sec:expe}, we evaluate our results by conducting experiments on MNIST and CIFAR-10, and a discussion follows in \cref{sec:discussion}.  

%% file: files/2_preliminaries.tex
\section{Preliminaries}\label{sec:background}

In this section, we first explain the main components of the SimCLR framework, detailing notation and underlying assumptions. Next, we introduce essential concepts from the PAC-Bayes theory.

\subsection{SimCLR framework}
The SimCLR framework \cite{chen2020simple} for contrastive learning involves several components: unlabeled data and augmentations, the model for representations, the contrastive loss, and for all practical purposes, some downstream prediction tasks. We describe these aspects with some theoretical considerations.
Specifically, the study of generalization requires an underlying probabilistic model for data. We consider a model inspired by CURL \cite{arora2019theoretical} that makes fewer assumptions about augmented pairs $(x,x^+)$ or negative samples.

\paragraph{Data distribution} Let $\mathcal{X}$ denote the input space and $\mathcal{D}_\mathcal{X}$ be a distribution on $\mathcal{X}$ from which unlabeled data is sampled. 
For any unlabeled instance $\bar{x}\sim \mathcal{D}_\mathcal{X}$, one generates random augmented views of $\bar{x}$. We use $\mathcal{A}(\cdot \mid \bar{x})$ to denote the distribution of augmented samples generated from $\bar{x}$. 

We define the distribution $\mathcal{S}$ as the process that generates a positive pair $(x, x^+)$ according to the following scheme:
(i) draw a sample $\bar{x} \sim \mathcal{D}_{{\mathcal{X}}}$;
(ii) draw two augmented samples $ x, x^{+} \sim \mathcal{A}(\cdot \mid \bar{x}).$
Note that we do not assume conditional independence of $x,x^+|\bar{x}$. The augmented views may also lie in a different space, outside $\mathcal{X}$, but this does not affect subsequent discussion.
The distribution $\mathcal{S}$ does not provide negative samples. This is actually computed in SimCLR from the positive pairs. In the subsequent presentation, we assume that one has access to i.i.d. sample of $n$ positive pairs $S\sim \mathcal{S}^n$. The dataset $S$ is randomly partitioned into equal-sized batches of $m$ i.i.d. positive pairs   $S_{batch}\subset S$. For each $x_i$, we use the tuple $(x_i,x_i^+,X_i^-)$ to compute the loss, where the set of negative samples $X_i^- = \bigcup_{j\neq i} \{x_j, x_j^+\}$ is taken from the batch $S_{batch}$ of size $m$ containing $x_i$.

\paragraph{Representation function} SimCLR learns a representation function \( f: \mathcal{X} \rightarrow \mathbb{S}^{d-1} \) that maps inputs to unit vectors in \( d \)-dimensional space. For convenience, we consider the set of representation functions to be parameterized by the weight space \( \mathcal{W} \subset \mathbb{R}^p \); for instance, it could define the output of a neural network architecture with \( p \) learnable parameters. Thus, each representation function is determined by its weight vector \( w \in \mathcal{W} \). 

\paragraph{Contrastive loss} 
In SimCLR, the similarity between the learned representations of two instances $f(x)$ and $f(x')$ is typically defined in terms of their cosine similarity, $\frac{f(x)^\top f(x')}{\Vert f(x)\Vert \; \Vert f(x')\Vert}$, where $\Vert\cdot\Vert$ is the Euclidean norm.
Since the feature representations are normalized, we express the similarity between two representations as \( f(x)^\top f(x') \). Let \( \tau \in (0,\infty) \) be a temperature parameter. To improve readability, we use the notation $\operatorname{sim}(x,x') = \exp \left(\frac{f(x)^\top f(x')}{\tau}\right)$ and define  $\displaystyle \ell_{\operatorname{cont}}(x, x^+, X) = -\log \frac{\operatorname{sim}(x,x^{+})}{\operatorname{sim}(x,x^{+}) + \sum_{x' \in X} \operatorname{sim}(x,x^{\prime})}$ . 
The empirical SimCLR loss over the dataset \( S \sim   \mathcal{S}^n\), denoted ${\widehat{L}}_{S}(f)$, is given by \cite{chen2020simple}: 
\begin{equation}
\label{eqn:empirical-SimCLR}
\frac{1}{n} \sum_{i=1}^n 
\frac{\ell_{\operatorname{cont}}(x_i, x_i^+, \bigcup_{j\neq i} \{x_j\}) +\ell_{\operatorname{cont}}(x_i^+, x_i, \bigcup_{j\neq i} \{x_j^+\}) }{2} = \frac{1}{n} \sum_{i=1}^n \ell(x_i, x_i^+, X_i^- ),
\end{equation}
where $\ell(x_i, x_i^+, X_i^- )$ denotes the above symmetrization of $\ell_{\operatorname{cont}}$.

\begin{remark}
\label{rem:first}
The SimCLR loss presented in \cref{eqn:empirical-SimCLR} differs slightly from the original loss from \cite{chen2020simple}, as it considers a subset of negative samples \( X_i^- \) instead of the full set \( X^- \) in each \(\ell_{\operatorname{cont}}\). This approach still uses all available negative samples while simplifying the theoretical derivations. Importantly, any results derived for this simplified SimCLR loss can be straightforwardly extended to the original SimCLR loss \cite{chen2020simple}. A detailed expression of the original loss is provided in the supplementary materials.
\end{remark}  

We define the population SimCLR loss, assuming a batch size of $m$, as
\begin{equation}
\label{eqn:population-SimCLR}
    {L}(f) = \underset{(x_i,x_i^+)_{i=1}^m\sim \mathcal{S}^m}{\mathbb{E}} \left[ \frac{1}{m} \sum_{i=1}^m \ell(x_i, x_i^+, X_i^- ) \right].
\end{equation}
Note that the population loss remains unchanged if it is computed over all $n$ pairs instead of a $m$-sized batch as the dependence of samples only occur within a batch. 

\paragraph{Evaluation of representations through downstream risk} Classification is a typical downstream task for contrastive representation learning. Following prior theoretical works \cite{arora2019theoretical,bao2022surrogate}, we assume that a linear classifier is trained on the representations. The multi-class classifier \( g: \mathcal{X} \rightarrow \mathbb{R}^C \) incorporates the learned representation \( f: \mathcal{X} \rightarrow \mathbb{R}^d \) (which remains frozen) and linear parameters \( W \in \mathbb{R}^{C \times d} \), defined by \( g(\cdot) := W f(\cdot) \). The linear classifier is learned by minimizing the supervised  cross-entropy loss of the multi-class classifier \( g \) expressed as:
\begin{equation}
\label{eqn:cross-entropy}
    L_{\mathrm{CE}}(f, W) ={\mathbb{E}}_{(x, y)\sim \mathcal{D}} \left[ -\log \frac{\exp \left(f(x)^{\top} w_y\right)}{\sum_{i=1}^C \exp \left(f(x)^{\top} w_i\right)} \right]
\end{equation}
where  \( W := [w_1 \cdots w_C]^{\top} \) and \( \mathcal{D} \) is the joint distribution on samples $x\in\mathcal{X}$ and labels $y \in \{1,\ldots,C\}$. 
We assume that $\mathcal{D}$ has the same marginal $\mathcal{D}_\mathcal{X}$ as the unlabeled data distribution. Additionally, we assume that any pair of positive samples \((x, x^{+})\) belongs to the same class, in line with the label consistency assumption presented in \cite{wang2022chaos}. This assumption is reasonable in practice, as data augmentations are not expected to change the class of the original sample.
We also define the top1 accuracy of the linear classifier, which can be evaluated as 
\begin{equation}
\label{eqn:top1-acc}
\texttt{top1} = 1 - R_{\text{top1}}(f, W),
\quad \text{where }
R_{\text{top1}}(f, W) = \mathbb{E}_{(x, y)\sim\mathcal{D}}  \left[ \mathbb{I}_{\{f(x)^{\top} w_y < \max\limits_{c \neq y} f(x)^{\top} w_c\}} \right]. 
\end{equation}
An alternative approach to evaluate learned representation $f(\cdot)$, without invoking a downstream problem, is in terms of the contrastive zero-one risk \cite{nozawa2020pac} defined as
\begin{align}
\label{eqn:contrastive01}
    R(f) &= \mathbb{E}_{(x_i, x_i^+)_{i=1}^m\sim\mathcal{S}^m}  \left[ \frac{1}{m} \sum_{i=1}^m \frac{1}{|X_i^-|}\sum_{x^{\prime} \in X_i^-}\mathbb{I}_{\left\{f(x_i)^{\top}f(x_i^+) < f(x_i)^{\top}f(x^{\prime})\right\}}\right]
    \\&= \mathbb{P}_{(x, x^+), (x^{\prime},x^{\prime+})\sim\mathcal{S}^2}  \left(f(x)^{\top}f(x^+) < f(x)^{\top}f(x^{\prime})\right).
    \nonumber
\end{align}
$R(f)$ evaluates the representation $f$ in terms of how often it embeds positive pairs further than negative samples. Since the SimCLR loss can be seen as a surrogate loss for the contrastive zero-one risk, risk certificates for the contrastive zero-one risk prove to be valuable \cite{nozawa2020pac}. 

\subsection{PAC-Bayes Theory}

PAC-Bayes theory, initially developed for simple classifiers \cite{seeger2002pac, catoni2007pac, germain2009pac}, has been extended to neural network classifiers in recent years \cite{dziugaite2017computing, perez2021tighter}, to contrastive learning \cite{nozawa2020pac}, and variational autoencoders \cite{cherief2022pac}. Here, we present the essential notions of PAC-Bayes theory and outline the most common PAC-Bayes generalization bounds. 

\paragraph{Notation} Let \( P \) denote a prior distribution and \( Q \) a posterior distribution over the parameter space \( \mathcal{W} \). In PAC-Bayes theory, the distance between the prior and posterior distributions is often quantified using the Kullback-Leibler (KL) divergence, defined as $\mathrm{KL}\left(Q \| Q^{\prime}\right)=\int\limits_{\mathcal{W}} \log \left(\frac{d Q}{d Q^{\prime}}\right) d Q .$
The binary KL divergence is
$\mathrm{kl}\left(q \| q^{\prime}\right) = q \log \left(\frac{q}{q^{\prime}}\right) + (1-q) \log \left(\frac{1-q}{1-q^{\prime}}\right)$, where \( q, q^{\prime} \in [0,1] \)---this measure quantifies the divergence between two Bernoulli distributions with parameters \( q, q^{\prime} \).  Let \( \mathcal{Z} \) denote an example space, \( \mathcal{D}_{\mathcal{Z}} \) the distribution over \( \mathcal{Z} \), and \( \ell_w: \mathcal{Z} \rightarrow [0,1] \) a loss function parameterized by \( w \in \mathcal{W} \). The risk \( L: \mathcal{W} \rightarrow [0,1] \) is defined as $L(w) = \underset{z \sim \mathcal{D}_{\mathcal{Z}}}{\mathbb{E}}[\ell_w(z)].$ Here, \( L(w) \) represents the expected value of \( \ell_w(z) \) under the distribution \( \mathcal{D}_{\mathcal{Z}} \). Let \( n \) be an integer, and the empirical risk for a dataset \( S = (z_1, \ldots, z_n) \in \mathcal{Z}^n \) is defined as
$\widehat{L}_S(w) = \frac{1}{n} \sum_{i=1}^n \ell_w(z_i).$ Here, \( \widehat{L}_S(w) \) computes the average loss \( \ell_w(z_i) \) over the dataset \( S \).

\paragraph{PAC-Bayes bounds} We extend the previously defined losses for a given weight \( w \) to losses for a given distribution \( Q \) over weights. Accordingly, the population loss of \( Q \) is defined as $L(Q) = \int_{\mathcal{W}} L(w) \, Q(dw).$ Similarly, the empirical loss of \( Q \) over a dataset \( S \) is given by $\widehat{L}_S(Q) = \int_{\mathcal{W}} \widehat{L}_S(w) \, Q(dw).$ The PAC-Bayes bounds relate the population loss \( L(Q) \) to the empirical loss \( \widehat{L}_S(Q) \) and other quantities through inequalities that hold with high probability. One of the fundamental results in PAC-Bayes theory is the PAC-Bayes-kl bound, from which various other PAC-Bayes bounds can be derived \cite{perez2021tighter}. 

\begin{theorem}[PAC-Bayes-kl bound \cite{perez2021tighter}]
\label{thm:kl-pb}
For any data-free distribution \( P \) over \( \mathcal{W} \) (i.e., prior), and for any \( \delta \in (0,1) \), with a probability of at least \( 1-\delta \) over size-\( n \) i.i.d. samples \( S \), simultaneously for all distributions \( Q \) over \( \mathcal{W} \) (i.e., posterior), the following inequality holds:
$$
    \mathrm{kl}\left(\widehat{L}_S(Q) \| L(Q)\right) \leq \frac{\mathrm{KL}\left(Q \| P\right)+\log \left(\frac{2 \sqrt{n}}{\delta}\right)}{n} .
$$
\end{theorem}

The PAC-Bayes-kl bound can be employed to derive the classic PAC-Bayes bound using Pinsker's inequality \( \mathrm{kl}(\hat{p} \| p) \geq 2(p-\hat{p})^2 \), as done in the following. 

\begin{theorem}[classic PAC-Bayes bound \cite{perez2021tighter}]
For any prior \( P \) over \( \mathcal{W} \), and any \( \delta \in (0,1) \), with a probability of at least \( 1-\delta \) over size-\( n \) i.i.d. random samples \( S \), simultaneously for all posterior distributions \( Q \) over \( \mathcal{W} \), the following inequality holds:
$$
    L(Q) \leq \widehat{L}_S(Q) + \sqrt{\frac{\mathrm{KL}\left(Q \| P\right) + \log \left(\frac{2 \sqrt{n}}{\delta}\right)}{2 n}} .
$$
\end{theorem}

In the context of this work, we define
$L(Q) := \int_{\mathcal{W}} L(f_w) \, Q(dw)$ as the SimCLR population loss \eqref{eqn:population-SimCLR} over a posterior distribution \( Q \), where \( f_w: \mathcal{X} \rightarrow \mathbb{S}^{d-1} \) represents the representation function parametrized by weights \( w \in \mathcal{W} \). For a sample \( S \sim \mathcal{S}^n \), the SimCLR empirical loss \eqref{eqn:empirical-SimCLR} over \( Q \) is defined as $\widehat{L}_S(Q) := \int_{\mathcal{W}} \widehat{L}_S(f_w) \, Q(dw)$. 
If $\widehat{L}_S(f_w)$ was an empirical mean of i.i.d. terms, then one could directly apply the above PAC-Bayes bounds. Unfortunately, the $\ell_{\operatorname{cont}}$ terms are dependent, and more effort is needed to derive useful bounds. 

%% file: files/3_theory.tex
\section{Risk Certificates for Contrastive Learning} \label{sec:contrib}

In this section, we present our main contributions: (1) using a PAC-Bayesian approach, we establish non-vacuous risk certificates for the SimCLR loss, accounting for the dependence across tuples $(x,x^+,X^-)$, (2) we derive a bound on the supervised loss tailored to SimCLR training, which incorporates data augmentation and temperature scaling, and (3) we extend our analysis to the contrastive zero-one risk and derive risk certificates.

\subsection{Risk Certificates for Contrastive Loss}
\label{sec:risk-certitificates}

Since the SimCLR loss is computed over a batch, it induces dependence across the summand $\ell$ in \eqref{eqn:empirical-SimCLR} through $X^-$, violating the assumption required in the PAC-Bayes-kl bound (\cref{thm:kl-pb}). Since the dataset $S$ is partitioned into $p = \frac{n}{m}$ i.i.d. batches, one simple workaround is to compute the PAC-Bayes bound over these i.i.d. batches instead of the usual $n$ i.i.d. samples. For a detailed overview of the bounds, please refer to the supplementary materials.  However, this weakens the risk certificate when there are few large batches. Alternatively, Nozawa et al. \cite{nozawa2020pac} suggest using $f$-divergence to address the non-i.i.d. characteristics, but their proposed bounds are vacuous. In this section, we present our first contribution: two PAC-Bayesian risk certificates specifically tailored to the SimCLR framework. 

\subsubsection{First PAC-Bayes certificate}
Our first approach involves acknowledging that, while the SimCLR loss does not strictly meet the required independence assumption, we can view the summand in \eqref{eqn:empirical-SimCLR} as \textit{almost} i.i.d. This idea can be formalized through a \textit{bounded difference assumption}.

\begin{definition}[bounded difference assumption]
    Let $A$ be some set and $\phi: A^n \rightarrow R$. We say $\phi$ satisfies the bounded difference assumption if $\exists c_1, \ldots, c_n \geq 0$ s.t. $\forall i, 1 \leq i \leq n$
$$
\sup _{x_1, \ldots, x_n, x_i^{\prime} \in A}\left|\phi\left(x_1, \ldots, x_i, \ldots, x_n\right)-\phi\left(x_1, \ldots, x_i^{\prime}, \ldots, x_n\right)\right| \leq c_i
$$
That is, if we substitute $x_i$ to $x_i^{\prime}$, while keeping other $x_j$ fixed, $\phi$ changes by at most $c_i$. 
\end{definition}

We can indeed prove that the SimCLR loss \eqref{eqn:empirical-SimCLR} satisfies the bounded difference assumption when one views the loss as a function of $\bar{x}_1,\ldots,\bar{x}_n \sim_{iid} \mathcal{D}_\mathcal{X}$ that generates the positive pairs (see \cref{lem:bda-simclr}). 
This assumption allows us to use McDiarmid's inequality \cite{mcdiarmid1989method}. Although one cannot directly incorporate this assumption into the proof of the PAC-Bayes-kl bound, it is possible to include McDiarmid's inequality in the proof of the PAC-Bayes bound presented in McAllester's work \cite{mcallester2003simplified}. We present our result in the following theorem. 

\begin{theorem}[Extended McAllester's PAC-Bayes Bound]
\label{thm:simclr-diarmid-pb}
 Let \(\delta \in (0,1)\) be a confidence parameter and $C = 4 + (m-1) \log \frac{(m-1)+  e^{2/\tau} }{m}$. With probability at least \(1-\delta\) over dataset \(S\sim \mathcal{S}^n\), for all \(Q\):
\[
L(Q) \leq \widehat{L}_S(Q) + C\sqrt{\frac{\text{KL}(Q \parallel P) + \log \frac{2n}{\delta}}{2(n-1)}}. 
\]
\end{theorem}

\begin{remark}
\label{rem:second}
One can observe that \cref{thm:simclr-diarmid-pb} holds for the original SimCLR loss as it satisfies the bounded difference assumption for $C = \frac{4}{\tau} + (m-1) \log \frac{2m-3 + 2e^{2/\tau}}{2m-1}$. The detailed proof is provided in the supplementary materials.
\end{remark}

The rest of the subsection proves \cref{thm:simclr-diarmid-pb}. We first show that the SimCLR loss satisfies the bounded difference assumption.

\begin{lemma}[bounded difference assumption for the SimCLR loss]
\label{lem:bda-simclr}
The SimCLR loss \eqref{eqn:empirical-SimCLR} can be expressed as
$ L_S(f) = \phi\left(\bar{x}_1, \ldots, \bar{x}_i, \ldots, \bar{x}_n\right)$,
where $\bar{x}_1, \ldots, \bar{x}_n \sim_{iid} \mathcal{D}_\mathcal{X}$ are the unlabeled samples that generate the positive pairs, $x_i,x_i^+\sim \mathcal{A}(\cdot|\bar{x}_i)$, and $\phi:\mathcal{X}^n\to\mathbb{R}$ is suitably defined to express $L_S(f)$.
Furthermore, the map $\phi$ satisfies the bounded difference assumption with $c_i = \frac{C}{n}$ where $C = \frac{4}{\tau} + (m-1) \log \frac{(m-1)+  e^{\frac{2}{\tau}} }{m}$.
\end{lemma}
\reversemarginpar

\begin{proof}[Proof of \cref{lem:bda-simclr}]
Let \( i \in [1, n] \) and $\bar S = \{\bar x_i\}_{i=1}^n$. We aim to bound the following quantity:
\[
\Delta \phi_{\bar S}(\bar{x}_i) = \phi\left(\bar{x}_1, \ldots, \bar{x}_i, \ldots, \bar{x}_n\right) - \phi\left(\bar{x}_1, \ldots, \bar{x}_i^{\prime}, \ldots, \bar{x}_n\right),
\]
where the notation emphasizes that the difference $\Delta \phi_{\bar S}(\bar{x}_i)$ arises from perturbing the $i$-th argument of the function $\phi$, while all other arguments $\bar{x}_j$ for $j \ne i$ are held fixed.

Without loss of generality, we assume that \( i \in \{1, \dots, m\} \), meaning \(\bar{x}_i\) belongs to the first batch $\bar S_m = \{\bar x_i\}_{i=1}^m$. The quantity \(\Delta \phi_{\bar S}(\bar{x}_i)\) can be split into two non-null terms:
\[
\Delta \phi_{\bar S}(\bar{x}_i) = \frac{1}{n} \left[ \delta^i_{\bar S_m}(\bar{x}_i) + \sum_{\substack{j=1 \\ j \neq i}}^m \delta^j_{\bar S_m}(\bar{x}_i) \right],
\]
where
\[
\begin{aligned}
    \delta^i_{\bar S_m}(\bar{x}_i) &:= \ell(x_i, x_i^+, X_i^-) - \ell(x_i^{\prime}, x_i^{\prime+}, X_i^-), \\
    \delta^j_{\bar S_m}(\bar{x}_i) &:= \ell(x_j, x_j^+, X_j^-) - \ell(x_j, x_j^+, \widetilde{X}_j^-),
\end{aligned}
\]
and \(\widetilde{X}_j^-\) is the set of negative samples perturbed with \(x_i^{\prime}\). Since the loss $\ell$ can be written as an average of two terms that play a symmetric role, upper-bounding \(\delta^i_{\bar S_m}(\bar{x}_i)\) reduces to upper-bounding
$\ell_{cont}(x_i, x_i^+, X_i) - \ell_{cont}(x_i^{\prime}, x_i^{\prime+}, X_i),$ where, for convenience, we denote \(X_i = \bigcup_{j \neq i} \{x_j\}\). Using the notation \(S(x, X) := \sum_{x^{\prime} \in X} \operatorname{sim}(x, x^{\prime})\), we have: 
\[
|\ell_{cont}(x_i, x_i^+, X_i) - \ell_{cont}(x_i^{\prime}, x_i^{\prime+}, X_i)| = \left|\log \frac{\operatorname{sim}(x_i^{\prime}, x_i^{\prime+})}{\operatorname{sim}(x_i, x_i^+)} + \log \frac{\operatorname{sim}(x_i, x_i^+) + S(x_i, X_i)}{\operatorname{sim}(x_i^{\prime}, x_i^{\prime+}) + S(x_i^{\prime}, X_i)} \right|\leq \frac{4}{\tau},
\]
as \(\forall x, y,\ e^{-\frac{1}{\tau}} \leq \operatorname{sim}(x, y) \leq e^{\frac{1}{\tau}}\). Similarly, the term \(\delta_j(\bar{x}_i)\) can also be written as an average of two terms. It suffices to upper-bound $\ell_{cont}(x_j, x_j^+, X_j) - \ell_{cont}(x_j, x_j^+, \widetilde{X}_j),$
where \(\widetilde{X}_j\) is the set of $X_j$ perturbed with \(x_i^{\prime}\). 
This can be upper-bounded as: 
$$
|\ell_{cont}(x_j, x_j^+, X_j) - \ell_{cont}(x_j, x_j^+, \widetilde{X}_j)| = \left|\log \frac{\kappa + \operatorname{sim}(x_j, x_i)}{\kappa + \operatorname{sim}(x_j, x_i^{\prime})}\right| \leq \log \frac{\kappa + a}{\kappa + b},
$$
where we define $\kappa = \operatorname{sim}(x_j, x_j^{+}) + \sum_{x^{\prime} \in N_i} \operatorname{sim}(x_j, x^{\prime})$ with $N_i = X_i\!\setminus\!\{x_i\}$, \(a = e^{1/\tau}\), and \(b = e^{-1/\tau}\).  Since \(a > b\), the function $x \mapsto \frac{x + a}{x + b}$ is strictly decreasing. This can be seen by computing its derivative: $f'(x) = \frac{b - a}{(x + b)^2} < 0$. Therefore, since $\kappa \ge (m-1) e^{-1/\tau}$ and the $\log$ function is increasing, we obtain 
$$
\log \frac{\kappa + a}{\kappa + b} \leq  \log \frac{(m-1)e^{-1/\tau} + e^{1/\tau}}{(m-1)e^{-1/\tau} + e^{-1/\tau}} = \log \frac{(m-1) + e^{2/\tau}}{m}.
$$

Combining these results, we obtain \(|\delta^i_{\bar S_m}(\bar{x}_i)| \leq \frac{4}{\tau}\) and \(|\delta^j_{\bar S_m}(\bar{x}_i)| \leq \log \frac{(m-1) + e^{\frac{2}{\tau}}}{m}\). This gives:
\[
|\Delta \phi_{\bar S}(\bar{x}_i)| \leq \frac{1}{n} \left( \frac{4}{\tau} + (m-1) \log \frac{(m-1) + e^{\frac{2}{\tau}}}{m} \right) = \frac{C}{n}.
\]
Taking the supremum concludes the proof.
\end{proof}

The rest of the proof of \cref{thm:simclr-diarmid-pb} follows the different steps of the proof presented in McAllester \cite{mcallester2003simplified}, while incorporating the bounded difference assumption. The first step involves showing that we can bound the quantity $\underset{S \sim \mathcal{S}^n}{\mathbb{E}}\left[e^{(n-1) h\left(L_S(f)-L(f)\right)}\right]$ for a certain non-negative and convex function $h$ (this corresponds to lemmas 5 and 6 in \cite{mcallester2003simplified}). 
\cref{lem:bounded-moment} below modifies lemma 6 in \cite{mcallester2003simplified} since we cannot use Hoeffding's inequality, whereas 
\cref{lem:small} is a minor modification of lemma 5 in \cite{mcallester2003simplified} that allows us to directly apply McDiarmid's inequality to prove \cref{lem:bounded-moment}. 

\begin{lemma}[adaptation of \cite{mcallester2003simplified}, Lemma 5]
\label{lem:small}
Let $X$ be a real valued random variable. If for $n \in \mathbb{N}^{*}$  and for \(x>0\), $\mathbb{P}(|X| \geq x) \leq 2e^{-nx^2}$, then $\mathbb{E}\left[e^{(n-1) X^2}\right] \leq 2n$.
\end{lemma}

\begin{proof}
The proof follows from McAllester's work \cite{mcallester2003pac}.  \\
Assume for \(x>0\),
$\mathbb{P}(|X| \geq x) \leq 2e^{-nx^2}$. Then the continuous density that maximizes $\mathbb{E}\left[e^{(n-1) |X|^2}\right]$ and satisfies the previous inequality is such that $\int_x^{\infty} f_{|X|}(u) \mathrm{d}u = 2e^{-nx^2}$ which gives \(f_{|X|}(u) = 4nu e^{-nu^2}\). We derive
$\mathbb{E}\left[e^{(n-1) |X|^2}\right] \leq \int_0^{\infty} e^{(n-1) u^2} f_{|X|}(u) \mathrm{d}u = 2n$.
\end{proof}

\begin{lemma}[adaptation of \cite{mcallester2003simplified}, Lemma 6]
\label{lem:bounded-moment}
Let \(h: x \mapsto \frac{2x^2}{C^2}\) and \(f \sim P\). We have
\[
\underset{S \sim \mathcal{S}^n}{\mathbb{E}}\left[e^{(n-1) h\left(L_S(f)-L(f)\right)}\right] \leq 2n.
\]
\end{lemma}

\begin{proof}
The bounded difference assumption allows us to derive the following McDiarmid's inequality: for \(\varepsilon>0\), $\mathbb{P}_S\left(|L(f) - L_S(f)| \geq \varepsilon\right) \leq 2 \exp\left(-\frac{2\varepsilon^2}{\sum_{i=1}^n c_i^2}\right)$. Using the change of variable  \(X = \frac{\sqrt{2}}{C}(L_S(f) - L(f))\), we observe that the McDiarmid's inequality gives: for $x>0$, $\mathbb{P}\left(\frac{\sqrt{2}}{C}|X| \geq x\right) \leq 2 \exp\left(-nx^2\right)$. Invoking \cref{lem:small} , we obtain $\mathbb{E}_S\left[e^{(n-1) h(X)}\right] \leq 2n$ where \(h: x \mapsto \frac{2x^2}{C^2}\). 
\end{proof}

Given the bound in \cref{lem:bounded-moment}, one can derive the PAC-Bayes bound as shown by McAllester \cite{mcallester2003simplified} (lemma 7--8). We adapt those results in the following statement.

\begin{lemma}[adaptation of \cite{mcallester2003simplified}, Lemma 7--8]
\label{lem:mcallester}
Let \(h\) be a non-negative and convex function. If for a fixed \(f \sim P\), the following inequality holds $\underset{S \sim \mathcal{S}^n}{\mathbb{E}}\left[e^{(n-1) h\left(L_S(f)-L(f)\right)}\right] \leq 2n$, then with probability at least \(1-\delta\) over i.i.d. dataset \(S\):
\[
\forall Q, \quad h(L_S(Q) - L(Q)) \leq \frac{\mathrm{KL}(Q \| P) + \log \frac{2n}{\delta}}{n-1}.
\]
\end{lemma}

\begin{proof}
Assume for a fixed \(f \sim P\), $\underset{S \sim \mathcal{S}^n}{\mathbb{E}}\left[e^{(n-1) h\left(L_S(f) - L(f)\right)}\right] \leq 2n$. This implies the bound $\mathbb{E}_S\left[\mathbb{E}_{f \sim P}\left[e^{(n-1) h\left(L_S(f) - L(f)\right)}\right]\right] \leq 2n$. Applying Markov's inequality, we obtain that with probability at least $1-\delta$ over $S$, $\mathbb{E}_{f \sim P}\left[e^{(n-1) h\left(L_S(f) - L(f)\right)}\right] \leq \frac{2n}{\delta}$.
Next, we use a shift of measure: $\mathbb{E}_{f \sim Q}\left[(n-1) h\left(L_S(f) - L(f)\right)\right] \leq \mathrm{KL}(Q \| P) + \log \mathbb{E}_{f \sim P}\left[e^{(n-1) h\left(L_S(f) - L(f)\right)}\right]$.
Combining the previous results, with probability at least \(1-\delta\) over dataset \(S\):
\[
\mathbb{E}_{f \sim Q}\left[(n-1) h\left(L_S(f) - L(f)\right)\right] \leq \mathrm{KL}(Q \| P) + \log \frac{2n}{\delta}.
\]
Since \(h\) is convex, applying Jensen's inequality finishes the proof:
\[
(n-1) h\left(L_S(Q) - L(Q)\right) \leq \mathrm{KL}(Q \| P) + \log \frac{2n}{\delta}.
\]
\end{proof}

Rewriting the statement of \cref{lem:mcallester} shows that, with probability at least \(1-\delta\) over i.i.d. dataset \(S\), $L(Q) \leq L_S(Q) + C\sqrt{\frac{\mathrm{KL}(Q \| P) + \log \frac{2n}{\delta}}{2(n-1)}}$ for all $Q$, completing the proof of \cref{thm:simclr-diarmid-pb}.

\subsubsection{Second PAC-Bayes certificate} 
While \cref{thm:simclr-diarmid-pb} uses the bounded difference assumption, an alternative approach 
would be to replace $\sum_{x'\in X}  \operatorname{sim}(x_i,x')$ in \eqref{eqn:empirical-SimCLR} by its expected value, conditioned on $x_i$, thereby making the terms in \eqref{eqn:empirical-SimCLR} independent. To implement this in the PAC-Bayes framework, we define an \(\varepsilon\)-modified SimCLR empirical loss \( L_S^{\prime}(Q) = \int_{\mathcal{W}} L_S^{\prime}(f_w) \, Q(dw) \) by replacing  \(\ell_{\text{cont}}\) in \eqref{eqn:empirical-SimCLR}  with an upper bound \(\ell^{\prime}_{\text{cont}}\) as
$$
    \ell_{\text{cont}}(x,x^+,X) \leq \ell^{\prime}_{\text{cont}}(x,x^+,X) =  -\log \frac{\operatorname{sim}(x, x^+)}{\operatorname{sim}(x, x^+) + \sum_{x^{\prime} \in X} \operatorname{sim}(x, x^{\prime}) + 2\varepsilon}.
$$
Modifying the loss by the above $\varepsilon$-factor allows us to bound it by an intermediate loss $\widetilde{L}_S(f)$, defined later, using a concentration bound on the negative samples. $\widetilde{L}_S(f)$ turns out to be a mean of i.i.d. terms, allowing us to directly apply the PAC-Bayes bound. The overall bound is summarized through the following theorem that provides a novel PAC-Bayes bound for the SimCLR population loss, extending the PAC-Bayes-kl bound using Hoeffding's inequality \cite{hoeffding1994probability}.

\begin{theorem}[Extended PAC-Bayes-kl Bound]\label{thm:simclr-kl-pb}
Let \( m \) denote the batch size and \(\delta \in (0,1)\). With probability at least \(1-\delta\) over dataset \(S\), for all \(Q\):
$$
L(Q) \leq \inf_{\alpha \in (0,1)} \left\{ B \cdot \mathrm{kl}^{-1} \left( \frac{1}{B} L_S'(Q) + \left(\frac{\delta}{2}\right)^{\frac{1-\alpha}{\alpha}},\frac{\mathrm{KL}(Q \| P) + \log \left(\frac{\sqrt{n}}{\delta}\right)}{n} \right) + \left(\frac{\delta}{2}\right)^{\frac{1}{\alpha}} \right\},
$$
where  $B := \frac{1}{\tau} + \log\left( m e^{\frac{1}{\tau}}\right) + \varepsilon$ for \(\varepsilon = \left(e^{\frac{1}{\tau}} - e^{-\frac{1}{\tau}} \right) \sqrt{\frac{m-1}{2\alpha}\log\frac{2}{\delta}}\) . 
\end{theorem}

\begin{remark}
\label{rem:third}
Observe that \Cref{thm:simclr-kl-pb} still holds for the original SimCLR loss by replacing \(\varepsilon\) with \(2\varepsilon\). This result follows from grouping terms involving augmented views of the same sample into one variable, \(\operatorname{sim}(x, x^{\prime}) + \operatorname{sim}(x, {x^{\prime}}^+)\), before applying the concentration bound. This ensures that the variables are independent while multiplying the bound by \(2\). \end{remark}

\normalmarginpar
The remainder of this subsection is dedicated to the proof of \cref{thm:simclr-kl-pb}.
First, we apply a Hoeffding's inequality to the term involving the negative samples, as stated below. 

\begin{lemma}[concentration bound on the negative samples]
\label{lem:one-kl}
Let $f\sim Q$ and recall the notation $S(x, X) := \sum\limits_{x^{\prime} \in X} \operatorname{sim}(x, x^{\prime})$ . For all \(\delta \in (0,1) \),  with \(\varepsilon = \left(e^{\frac{1}{\tau}} - e^{-\frac{1}{\tau}}\right) \sqrt{\frac{m-1}{2} \log \frac{1}{\delta}} \), 
the following concentration bound holds: $\mathbb{P}\left( S(x, X) - \mathbb{E}[S(x, X)] \geq \varepsilon \mid x \right) \leq \delta.$
\end{lemma}

\begin{proof}
Let \(\delta \in (0,1) \). Conditioned on $x$, we have a sum of independent and bounded variables, as each variable is lower-bounded by \(e^{-\frac{1}{\tau}} \) and upper-bounded by \(e^{\frac{1}{\tau}} \). Hence, Hoeffding's inequality gives: for all \( \varepsilon > 0 \),
$\mathbb{P}\left( S(x, X) - \mathbb{E}[S(x, X)] \geq \varepsilon \mid x \right) \leq \delta$,  where we set \( \delta := \exp \left( -\frac{2\varepsilon^2}{(m-1)c^2} \right) \)  for $c = e^{\frac{1}{\tau}} - e^{-\frac{1}{\tau}}$. Deriving the expression of  $\varepsilon$ finishes the proof.
\end{proof}

Recall the SimCLR population loss $L(f)$ defined in \cref{eqn:population-SimCLR}. Define the intermediate loss $\widetilde{L}(f)$  which is a mean of i.i.d. terms obtained by replacing  \(\ell_{\text{cont}}\) with the loss \(\tilde\ell_{\text{cont}}\) given by
$\tilde{\ell}_{\text{cont}}(x, x^+) := -\log \frac{\operatorname{sim}(x, x^+)}{\operatorname{sim}(x, x^+) +\mathbb{E} \left[S(x, X) \mid x\right] + \varepsilon}$  for  a suitable $\varepsilon>0$. We obtain the following 
\begin{equation}
\label{eq:inter-loss}
    \widetilde{L}(f) := \mathbb{E}_{(x, x^+) \sim \mathcal{S}} \left[ \frac{\tilde{\ell}_{\text{cont}}(x, x^+) + \tilde{\ell}_{\text{cont}}(x^+, x)}{2} \right] = \mathbb{E}_{(x, x^+) \sim \mathcal{S}} \left[ \tilde{\ell}(x, x^+) \right].
\end{equation}
Using \cref{lem:one-kl}, the following lemma shows that $L(f)$ can be upper-bounded by $\widetilde{L}(f)$.  

\begin{lemma}[upper-bound by an intermediate loss]
\label{lem:inter}
Let $f \sim Q$ and $\widetilde{L}(f)$ defined in \cref{eq:inter-loss}. Let  \(\delta \in (0,1) \) and \( B_{\ell} = \frac{2}{\tau} + \log m \). Then, we have
\begin{equation}
    L(f) \leq \widetilde{L}(f) + B_{\ell} \delta.
\end{equation}
\end{lemma}
\begin{proof}

Recall that \(\ell_{\operatorname{cont}}(x, x^+, X^-)\) can be written as an average of two terms that play a symmetric role. Hence, we will first upper-bound \(\ell_{\operatorname{cont}}(x, x^+, X)\) where \(X = \bigcup_{x^{\prime} \neq x} \{x^{\prime}\}\). Observe that \(\ell_{\operatorname{cont}}(x, x^+, X)\) can be split into two terms using $\varepsilon$ defined in \cref{lem:one-kl}:
$$
\ell_{\operatorname{cont}}(x, x^+, X)\leq \tilde{\ell}_{\text{cont}}(x, x^+) \mathbb{I}_{\{ S(x, X) < \mu(x) + \varepsilon \}}  + \ell_{\operatorname{cont}}(x, x^+, X) \mathbb{I}_{\{ S(x, X) \geq \mu(x) + \varepsilon \}},
$$
where $\mu(x) = \mathbb{E} \left[S(x, X) \mid x\right]$. 
Noting that $\ell_{\operatorname{cont}}(x, x^+, X) \leq B_{\ell}$ and \(\mathbb{I}_{\{ \cdot \}} \leq 1\), we obtain
$\ell_{\operatorname{cont}}(x, x^+, X) \leq \tilde{\ell}_{\text{cont}}(x, x^+) + B_{\ell} \mathbb{I}_{\{ S(x, X) \geq \mu(x) + \varepsilon \}}$. We derive:
\begin{equation}
\label{eq:subloss-inter}
    \ell(x, x^+, X^-) \leq \tilde{\ell}(x, x^+) + B_{\ell} \frac{\mathbb{I}_{\{ S(x, X) \geq \mu(x) + \varepsilon \}}+ \mathbb{I}_{\{ S(x^+, X^{\prime}) \geq \mu(x^+) + \varepsilon \}}}{2},
\end{equation}
where \(X^{\prime} = \bigcup_{x^{\prime} \neq x} \{{x^{\prime}}^{+}\}\). Using \cref{lem:one-kl} and the expression of $L(f)$ combined with \cref{eq:subloss-inter}, 
\begin{equation}
L(f) \leq \underset{(x_i,x_i^+)_{i=1}^m\sim \mathcal{S}^m}{\mathbb{E}} \left[ \frac{1}{m}\sum_{i=1}^m \tilde{\ell}(x_i, x_i^+) \right] + B_{\ell} \delta \, ,
\end{equation}
which finishes the proof.
\end{proof}

Define the empirical intermediate loss as
\begin{equation}
\label{eq:empirical-inter-loss}
    \widetilde{L}_S(f) = \frac{1}{m}\sum_{i=1}^m \tilde{\ell}(x_i, x_i^+).
\end{equation}Since $\widetilde{L}_S(f)$ is an average of i.i.d. terms, the following result is a direct application of the PAC-Bayes-kl bound (\cref{thm:kl-pb}) after rescaling of the loss function to the interval \([0,1]\).

\begin{corollary}[PAC-Bayes-kl bound for the intermediate loss]
\label{lem:kl-inter}
 For \( S \sim \mathcal{S}^n \), let $\widetilde{L}(f)$ and $\widetilde{L}_S(f)$ be defined as in \cref{eq:inter-loss} and \cref{eq:empirical-inter-loss} respectively, and define the extended losses:  $\widetilde{L}(Q) = \mathbb{E}_{f \sim \mathcal{Q}} \left[ \widetilde{L}(f) \right]$ and $\widetilde{L}_S(Q) = \mathbb{E}_{f \sim \mathcal{Q}} \left[ \widetilde{L}_S(f) \right]$. Let $B := \frac{1}{\tau} + \log( me^{\frac{1}{\tau}} + \varepsilon)$. Given a prior \(P\) over \(\mathcal{F}\) and \(\delta \in (0,1)\), with probability at least \(1-\delta\) over i.i.d. samples \(S \sim \mathcal{S}^n\), for all \(Q\) over \(\mathcal{F}\),
$$
\frac{1}{B}\widetilde{L}(Q) \leq \mathrm{kl}^{-1} \left( \frac{1}{B}\widetilde{L}_S(Q), \frac{\mathrm{KL}\left(Q \| P\right) + \log \left(\frac{2 \sqrt{n}}{\delta}\right)}{n} \right).
$$
\end{corollary}

The next lemma allows to upper-bound the empirical intermediate loss by a term similar to the SimCLR empirical loss, introduced earlier as the \( \varepsilon \)-modified  SimCLR loss $L_S^{\prime}(Q)$.

\begin{lemma}[upper-bound on the intermediate empirical loss]
\label{lem:upper-inter}
Let $\alpha \in (0,1)$ and \(\delta \in (0,1) \). With probability at least \( 1 - \delta^{\alpha} \) over dataset \( S \):
$$
\widetilde{L}_S(Q) \leq B \delta^{1-\alpha} + L_S^{\prime}(Q).
$$
\end{lemma}

\begin{proof}[Proof of \cref{lem:upper-inter}]
To prove \cref{lem:upper-inter}, we first need the following useful lemma.
\begin{lemma}
\label{lem:alpha}
 Let $\alpha \in (0,1)$ and $\delta \in (0,1)$. If for any event $A(f, x, X)$ with $f \sim Q$, we have   ${\mathbb{P}}_{X}\left[ A(f, x, X) \mid f, x \right] \leq \delta$, then, with probability at least $1-\delta^{\alpha}$ over dataset $S\sim \mathcal{S}^n$,
$$\frac{1}{n} \sum_{i=1}^{n}\underset{f \sim Q}{\mathbb{E}}\left[\mathbb{I}_{\{ A(f, x_i, X_i) \}}\mid x_i, X_i^-\right] \leq \delta^{1-\alpha}.$$
\end{lemma}
\begin{proof}[Proof of \cref{lem:alpha}]
For ease of notation, define $U = \mathbb{I}_{\{  A(f, x, X)\}}$ with $f \sim Q$. Assume for any augmented sample $x$,  ${\mathbb{P}}_{X}\left[ A(f, x, X) \mid f, x \right] \leq \delta$. Then, using Markov's inequality, we derive: 
\begin{align*}
{\mathbb{P}}_S\left[ \frac{1}{n} \sum_{i=1}^{n}\underset{f \sim Q}{\mathbb{E}}\left[U \mid x_i, X_i \right] > \delta^{1-\alpha} \right] & \leq \frac{1}{\delta^{1-\alpha}}{\mathbb{E}}_S\left[ \frac{1}{n} \sum_{i=1}^{n}\underset{f \sim Q}{\mathbb{E}}\left[U \mid x_i, X_i\right]  \right] \\
& = \frac{1}{\delta^{1-\alpha}}\frac{1}{n} \sum_{i=1}^{n}\mathbb{E}_{x_i, X_i}\left[ {\mathbb{E}_f}\left[U\mid x_i, X_i\right]  \right] \\
& = \frac{1}{\delta^{1-\alpha}}\frac{1}{n} \sum_{i=1}^{n}\mathbb{E}_{f, x_i, X_i}\left[U \right] \\
& = \frac{1}{\delta^{1-\alpha}}\frac{1}{n} \sum_{i=1}^{n}\mathbb{E}_{f, x_i}\left[\mathbb{E}_{ X_i}\left[U \mid f, x_i\right] \right] \\
& \leq \frac{1}{\delta^{1-\alpha}}\frac{1}{n} \sum_{i=1}^{n}\mathbb{E}_{f, x_i}\left[\delta \right] = \delta^{\alpha}.
\end{align*}
\end{proof}

Define $ A(f, x, X) := \{S(x, X)  - \mu(x) \leq -\varepsilon\}$. Similarly to the inequality from \cref{lem:one-kl}, we can derive, for an augmented sample \( x \), $\mathbb{P}\left( S(x, X) - \mu(x) \leq - \varepsilon \mid x \right) \leq \delta$. Next, using similar arguments to the proof of lemma \cref{lem:inter},  we derive an upper bound on $\tilde{\ell}(x, x^+)$:
$$
\tilde{\ell}(x, x^+) \leq B \frac{\mathbb{I}_{\{ S(x, X) + \varepsilon  \leq \mu(x) \}} + \mathbb{I}_{\{S(x^+, X^{\prime}) + \varepsilon  \leq \mu(x^+) \}}}{2} + \ell^{\prime}(x, x^+, X^-).
$$
Plugging this into the intermediate empirical loss, we obtain:
$$\widetilde L_S(Q) \leq   \frac{B}{n} \sum_{i=1}^{n} \frac{1}{2}\left(\underset{f \sim Q}{\mathbb{E}}\left[\mathbb{I}_{\{ A(f, x_i, X_i) \}}\mid x_i, X_i\right] + \underset{f \sim Q}{\mathbb{E}}\left[\mathbb{I}_{\{ A(f, x_i^+, X_i^{\prime}) \}}\mid x_i^+, X_i^{\prime}\right] \right)+ L_S^{\prime}(Q).$$
Invoking \cref{lem:alpha},  we obtain: with probability at least \( 1 - \delta^{\alpha} \) over dataset \( S \):
$$
\widetilde{L}_S(Q) \leq B \delta^{1-\alpha} + L_S^{\prime}(Q),
$$which finishes the proof.
\end{proof}

Let $\alpha \in (0,1)$. We will now combine all previously established lemmas to finish the proof of \cref{thm:simclr-kl-pb}. Given that \( B_{\ell} \leq B \) and for $\delta_1 \in (0,1)$, \cref{lem:inter} can be rewritten as $L(Q) \leq \widetilde{L}(Q) + B\delta_1$. Next, combining this inequality with \cref{lem:kl-inter}, we obtain for $\delta_2 \in (0,1)$: with probability at least \(1-\delta_2\) over training i.i.d. samples \(S \sim \mathcal{S}^n\), for all \(Q\) over \(\mathcal{F}\),
$$
\frac{1}{B}\widetilde{L}(Q) \leq \mathrm{kl}^{-1} \left( \frac{1}{B}\widetilde{L}_S(Q), \frac{\mathrm{KL}\left(Q \| P\right) + \log \left(\frac{2 \sqrt{n}}{\delta_{2}}\right)}{n} \right) + \delta_{1}.
$$
Using a union bound argument, and given the property of \( \mathrm{kl}^{-1} \) as a monotonically increasing function of its first argument when fixing the second argument \cite{perez2021tighter}, we can now combine the previous bounds with \cref{lem:upper-inter}: with probability at least \( 1 - \delta_{2} -  \delta_1^{\alpha}\) over dataset $S$, for all \(Q\) over \(\mathcal{F}\),
$$
\frac{1}{B} L(Q) \leq \mathrm{kl}^{-1} \left( \frac{1}{B} L_S^{\prime}(Q) + \delta_1^{1-\alpha}, \frac{\mathrm{KL}\left(Q \| P\right) + \log \left(\frac{2\sqrt{n}}{\delta_{2}}\right)}{n} \right) + \delta_1.
$$
Let $\delta \in (0,1)$ and set \(\ \delta_1^{\alpha} = \frac{\delta}{2}\), \(\delta_{2} = \frac{\delta}{2}\). Taking the infimum over $\alpha$ finishes the proof.

\subsection{Bound on the Downstream Classification Loss}\label{sec:downstream}

Bao et al. \cite{bao2022surrogate} derive a bound on the downstream supervised loss, addressing a key limitation of the earlier bound by Arora et al. \cite{arora2019theoretical}, whose bound grows exponentially with the number of negative samples. However, Bao et al. \cite{bao2022surrogate} employ a simplified framework where positive pairs are not generated through data augmentation, and the contrastive loss is computed without batching and temperature scaling. Empirical results, notably from the SimCLR framework, have demonstrated superior performance over this basic approach \cite{chen2020simple}. 

In the following theorem, we extend Bao et al.'s bound \cite{bao2022surrogate} to accommodate the SimCLR framework. Specifically, we show that the loss can be computed over a batch and that, by leveraging an argument from Wang et al. \cite{wang2022chaos}, the bound holds when positive pairs correspond to augmented views. Additionally, we introduce a novel extension of the bound that incorporates temperature scaling. 

\begin{theorem}[Extended Surrogate Gap]
\label{thm:downstream}
Consider the cross-entropy loss $L_{\mathrm{CE}}(f, W)$ in \eqref{eqn:cross-entropy} and denote the class distribution by $\boldsymbol{\pi}:=[\mathbb{P}(Y=c)]_{c \in \mathcal{C}}$. For all \( f: \mathcal{X} \rightarrow \mathbb{S}^{d-1} \), the following inequality holds:
\begin{equation*}
    \min _{W \in \mathbb{R}^{C \times d}} L_{\mathrm{CE}}(f, W) \leq  \min 
    \big\{
  \beta(f, \sigma),~ 
   \tau \beta(f, \sigma) + \alpha
    \big\}
\end{equation*}
with \(  \beta(f, \sigma) = \frac{\sigma}{\tau} + L(f) + \Delta \) for  
$\Delta = \log\left(\frac{C^2\pi^* }{m-1}\operatorname{cosh}^2\left(\frac{1}{\tau}\right)\right)$ where $\pi^* := \max_{c \in \mathcal{C}} \pi(c)$. 

Furthermore, $\alpha$ is given by $\alpha = \log(C) + \min\{0, \log(C\pi^*\operatorname{cosh}^2(1)) - \tau \Delta\}$ and \( \sigma = {\mathbb{E}}_{(x, y)} \left[ \| f(x)-\mu_{y} \|_2 \right] \leq 2\) represents the intra-class feature deviation.
\end{theorem}
\begin{remark}
\label{rem:fourth}
Practical implementations often consider a deterministic or learnable projection of the representation $f(x)$. After contrastive training, the projection head is removed, and only the backbone features are used for downstream tasks. Notably, \cref{thm:downstream} can be extended to include a simple projection head of the form $(I_k, 0)$. Let \( f(x) \in \mathbb{R}^d \)  denote the backbone features and \( f_1(x) \in \mathbb{R}^{k} \) the projection head output, where typically \( d \geq k \). We can generalize \cref{thm:downstream} by considering the contrastive loss on the projected features $L(f_1)$ instead of $L(f)$ and by modifying the supervised loss $L_{\mathrm{CE}}(f, W)$ as follows:
$$
\underset{x, y \sim \mathcal{D}}{\mathbb{E}} \left[-\log \frac{\exp \left(f_1(x)^{\top} w^{(1)}_y + f_2(x)^{\top} w^{(2)}_y\right)}{\sum_{i=1}^C \exp \left(f_1(x)^{\top} w^{(1)}_i + f_2(x)^{\top} w^{(2)}_i\right)}\right]
$$
where \( f_2(x) \in \mathbb{R}^{d-k} \) is defined such that \( f(x) = \begin{bmatrix} f_1(x) \\ f_2(x) \end{bmatrix} \) and \( W = \begin{bmatrix}W^{(1)} \\ W^{(2)} \end{bmatrix} \).
\end{remark}

\begin{remark}  
\label{rem:fifth}
Note that \cref{thm:downstream} can be applied to the original SimCLR loss by replacing \(m-1\) with \(2(m-1)\).  When the class distribution is uniform, $\pi^*$ simplifies to $1/C$.
\end{remark}

The rest of the subsection proves \cref{thm:downstream}. First, observe that the SimCLR population loss defined in \cref{eqn:population-SimCLR} can be rewritten as $L(f) = {\mathbb{E}}_{(x_i,x_i^+)_{i=1}^m\sim \mathcal{S}^m} \left[ \frac{1}{m} \sum_{i=1}^m \ell_{\operatorname{cont}}(x_i, x_i^+, X_i) \right]$ where \(X_i = \bigcup_{j \neq i} \{x_j\}\) denotes the set of negative samples. The first part of the theorem builds on the proof from Bao et al. \cite{bao2022surrogate}, but instead of their simple framework, we enhance the applicability of the proof for the SimCLR framework, where the contrastive loss is computed over a batch and includes temperature scaling. Additionally, to address the fact that the positive pair is generated through data augmentation, we incorporate an argument from Wang et al.'s proof \cite{wang2022chaos}. The following lemma presents the first part of the theorem.  

\begin{lemma}
\label{lem:first-bound}
For all \( f: \mathcal{X} \rightarrow \mathbb{S}^{d-1} \), the following inequality holds:
\[
\min_{W \in \mathbb{R}^{C \times d}} L_{\mathrm{CE}}(f, W) \leq \frac{\sigma}{\tau} + L(f) + \Delta.
\]
\end{lemma}

\begin{proof}
Let $\operatorname{LSE}(\mathbf{z}) := \log \left( \sum_{j} \exp \left(z_j \right) \right)$ where $\mathbf{z} := \left\{ {f(x)^{\top}f(x^{\prime})/\tau}\right\}_{x^{\prime} \in X}$. For convenience, we define $X$ as the set of negative samples corresponding to the anchor sample $x$, with the total number of negative samples denoted by $K = m - 1$. The elements of $X$ are indexed by $i = 1, \ldots, K$. Denoting the class-conditional distribution by $\mathcal{D}_c:=\mathbb{P}(X \mid$ $Y=c)$ for each $c \in \mathcal{C}$, we define $\mu_c =\mathbb{E}_{\mathbf{x} \sim \mathcal{D}_c}[f(x)]$ and note that $\|\mu_c\| \leq 1$. We have
\begin{align}
\label{eq:lem-first-part-a}
 L(f)&\overset{(a)}{\geq} - \underset{{S \sim \mathcal{S}^m}}{\mathbb{E}} \left[ \frac{1}{m} \sum_{i=1}^m \frac{f({x_i})^{\top}f({x_i}^+)}{\tau} \right] + \underset{{S \sim \mathcal{S}^m}}{\mathbb{E}} \left[ \frac{1}{m} \sum_{i=1}^m \log \sum_{x^{\prime} \in X_i} \exp \left(\frac{f({x_i})^{\top}f({x^{\prime}})}{\tau}\right) \right] \\
 \label{eq:lem-first-part-b}
& \overset{(b)}{=}  - \mathbb{E}_{(x, x^+) \sim \mathcal{S}} \left[ \frac{f({x})^{\top}f({x}^+)}{\tau} \right] + \frac{1}{m} \sum_{i=1}^m \underset{S \sim \mathcal{S}^m}{\mathbb{E}} \left[ \operatorname{LSE}(\mathbf{z})\right]\\
\label{eq:lem-first-part-c}
& \overset{(c)}{\geq} - \frac{\mathbb{E}_{(x, y)} \left[ \| f(x) - \mu_{y} \|_2 \right]}{\tau} - \mathbb{E}_{(x, y)} \left[ \frac{f(x)^{\top} \mu_{y}}{\tau} \right] + \mathbb{E}_x \mathbb{E}_{X} \left[ \operatorname{LSE}(\mathbf{z}) \right]\\
\label{eq:lem-first-part-d}
& \overset{(d)}{=} - \frac{\sigma}{\tau} - \mathbb{E}_{(x, y)} \left[ \frac{f(x)^{\top} \mu_{y}}{\tau} \right] + \mathbb{E}_x \mathbb{E}_{\substack{\{(x_i^{\prime}, y_i^{\prime})\}_{i=1}^K}} \left[ \operatorname{LSE}(\mathbf{z}) \right] \\
\label{eq:lem-first-part-e}
& \overset{(e)}{\geq} - \frac{\sigma}{\tau} - \mathbb{E}_{(x, y)} \left[ \frac{f(x)^{\top} \mu_{y}}{\tau} \right] + \mathbb{E}_x\mathbb{E}_{\substack{\{y_i^{\prime}\}_{i=1}^K}} \left[ \operatorname{LSE}(\mathbf{z}^{\mu}) \right]
\end{align}

where (a) relies on $\exp \left(f({x})^{\top}f({x}^+)\right) \geq 0$, (b) follows from rewriting the expectations and using the definition of the $\operatorname{LSE}$ function (c) uses the argument from Wang et al. : using the Cauchy-Schwarz inequality, assuming $y$ is the label of the positive pair $(x, x^+)$, we can show for any $\mu_y$, $\left| f(x)^{\top} \frac{f(x^+) - \mu_{y}}{\left\| f(x^+) - \mu_{y} \right\|} \right| \leq 1$  \cite{wang2022chaos}, (d) uses the definition of the intra-class feature deviation $\sigma := \mathbb{E}_{(x, y)} \left[ \| f(x) - \mu_{y} \|_2 \right]$, (e) uses the Jensen's inequality and the convexity of the $\operatorname{LSE}$ function where we define $\mathbf{z}^{\mu} := \left\{ {f(x)^{\top}\mu_{y^{\prime}}/\tau} \right\}_{y' \in Y}$ where we denote $Y = \substack{\{y_i^{\prime}\}_{i=1}^K}$.

Next, we need to lower-bound the quantity $\mathbb{E}_{Y} \left[ \operatorname{LSE}(\mathbf{z}^{\mu}) \right]$. We will use the following lemma.
\begin{lemma}[Bao et al. \cite{bao2022surrogate}, lemma 3]
\label{lem:bao}
Let $N \in \mathbb{N}^{*}$ and $L \in \mathbb{R}^+$. For $\mathbf{z} \in [-L^2, L^2]^N$,
\[
2 \log N \leq \mathrm{LSE}(\mathbf{z}) + \mathrm{LSE}(-\mathbf{z}) \leq 2 \log \left(N \cosh (L^2)\right).
\]
\end{lemma}
We apply this lemma for $N=K$ and $L={1/\sqrt{\tau}}$. We obtain

\begin{align}
\label{eq:lem-second-part-a}
 \mathbb{E}_{Y} \left[ \operatorname{LSE}(\mathbf{z}^{\mu}) \right] & \overset{(a)}{\geq}  -\mathbb{E}_{Y} \left[ \operatorname{LSE}(-\mathbf{z}^{\mu}) \right] + 2\log(K) \\
\label{eq:lem-second-part-b}
& \overset{(b)}{\geq}  - \log \sum_{y^{\prime} \in Y} \mathbb{E}_{\substack{y^{\prime}}} \left[ \exp (-f(x)^{\top}\mu_{y^{\prime}}/\tau) \right] + 2\log(K)\\
\label{eq:lem-second-part-c}
& \overset{(c)}{=} -\log K\sum_{c \in \mathcal{C}}\exp(-f(x)^{\top}\mu_{c}/\tau) \pi(c) + 2\log(K) \\
\label{eq:lem-second-part-d}
&  \overset{(d)}{\geq} -\log {K\pi^*} \sum_{c \in \mathcal{C}} \exp(-f(x)^{\top}\mu_{c}/\tau) + 2\log(K) \\
\label{eq:lem-second-part-e}
&  \overset{(e)}{=} -\operatorname{LSE}\left( \left\{ -f(x)^{\top}\mu_{c} /\tau\right\}_{c \in \mathcal{C}} \right) + \log(K) - \log(\pi^*) \\
\label{eq:lem-second-part-f}
& \overset{(f)}{\geq} \operatorname{LSE}\left( \left\{ f(x)^{\top}\mu_{c}/\tau \right\}_{c \in \mathcal{C}} \right) - 2\log\left(C \cosh\frac{1}{\tau}\right) + \log(K) - \log(\pi^*) \\
\label{eq:lem-second-part-g}
& \overset{(g)}{=} \operatorname{LSE}\left( \left\{ f(x)^{\top}\mu_{c}/\tau \right\}_{c \in \mathcal{C}} \right) - \Delta\, ,
\end{align}
where (a) directly comes from \cref{lem:bao}, (b) applies Jensen’s inequality to the convex function $x \mapsto -\log x$ and then the linearity of expectation, (c) uses the distribution $\pi(c)$ over the classes and $\mathcal{C}$ represents the set of classes, (d) uses the bound $\pi(c)\leq \pi^*$ for all $c \in \mathcal{C}$ and the monotonicity of $-\log$, (e) uses the definition of the $\operatorname{LSE}$ function, (f) is another application of \cref{lem:bao}, and (g) uses the definition $\Delta := \log \left( \frac{C^2\pi^*}{K} \operatorname{cosh}^2(\frac{1}{\tau}) \right)$.

We can now combine the previous inequality and recognizing the cross-entropy loss of a linear classifier for  \( W^{\mu/\tau} := [\frac{\mu_{1}}{\tau} \cdots \frac{\mu_{C}}{\tau}]^{\top} \), we have
\begin{align*}
 L(f)& {\geq} - \frac{\sigma}{\tau} - \mathbb{E}_{(x, y)} \left[ \frac{f(x)^{\top} \mu_{y}}{\tau} \right] + \mathbb{E}_x\operatorname{LSE}\left( \left\{ \frac{f(x)^{\top}\mu_{c}}{\tau} \right\}_{c \in \mathcal{C}} \right) - \Delta \\
 & = - \frac{\sigma}{\tau}  + L_{\mathrm{CE}}(f, W^{\mu/\tau}) - \Delta
\end{align*}
Taking the minimum over all linear classifiers finishes the proof.
\end{proof}
The second part of the theorem is an extension of \cref{lem:first-bound} obtained by removing the temperature scaling from the $\operatorname{LSE}$ term. The result is stated in the following lemma.

\reversemarginpar

\begin{lemma}
\label{lem:second-bound} 
For all \( f: \mathcal{X} \rightarrow \mathbb{S}^{d-1} \), the following inequality holds:
\[
\min_{W \in \mathbb{R}^{C \times d}} L_{\mathrm{CE}}(f, W) \leq \sigma + \tau L(f) + \tau \Delta + \alpha,
\]
where \( \alpha = \log(C) + \min \left\{0,  \log(C\pi^*\operatorname{cosh}^2 (1))-\tau \Delta \right\}\).
\end{lemma}
\begin{proof}[Sketch of Proof]
For a detailed proof of the lemma, we refer the reader to the supplementary materials. We first derive a useful inequality that show that the 
$\operatorname{LSE}$ function acts as a smooth maximum. For $N \in \mathbb{N}^{*}$, $\mathbf{x} \in \mathbb{R}^N$ with $\mathbf{x} = (x_1, \ldots, x_N)$, and for any $t > 0$,
\begin{equation}
\label{eqn:smooth-max}
    t \max_i x_i < \operatorname{LSE}(t\mathbf{x}) \leq t \max_i x_i + \log (N).
\end{equation}
The inequality is verified by observing that if $m=\max _i x_i$, then we have the following inequality $\exp (tm) \leq \sum_{i=1}^N \exp \left(tx_i\right) \leq N \exp (tm)$. Applying the logarithm gives the result. 

Define \( W^{\mu} := [\mu_{1} \cdots \mu_{C}]^{\top} \). The bound presented in \cref{lem:second-bound} follows from two applications of \eqref{eqn:smooth-max} at different steps of the proof of \cref{lem:first-bound}. On one hand, applying \eqref{eqn:smooth-max} to the term $\operatorname{LSE}(\mathbf{z})$ introduces a factor of \(\log(K)\) . With \(\mathbf{u} := \left\{f(x)^{\top}f(x^{\prime})\right\}_{x^{\prime} \in X}\), this yields \(\operatorname{LSE}(\mathbf{z}) \geq \frac{1}{\tau} \left( \operatorname{LSE}(\mathbf{u}) - \log(K) \right)\). Following a proof similar to  \cref{lem:first-bound} with $L=1$, we obtain a lower bound on \(\operatorname{LSE}(\mathbf{u})\): \(\operatorname{LSE}(\mathbf{u}) \geq \operatorname{LSE}\left( \left\{ f(x)^{\top}\mu_{c} \right\}_{c \in [C]} \right) - \log \left( \frac{C^2\pi^*}{K} \operatorname{cosh}^2(1) \right)\). Combining this with the previous bound,  we derive 
\begin{equation}
\label{eq:first-smooth-max}
    L_{\mathrm{CE}}(f, W^{\mu}) \leq \sigma + \tau L(f) + \tau \Delta + \log (C) + \log(C\pi^*\operatorname{cosh}^2(1)) - \tau \Delta.
\end{equation}
On the other hand, applying \eqref{eqn:smooth-max} to the term $\operatorname{LSE}(\mathbf{z}^c)$ with $\mathbf{z}^{c} := \left\{ {f(x)^{\top} \mu_{c}/\tau} \right\}_{c \in \mathcal{C}}$  introduces a factor of \(\log(C)\), leading to $\tau \operatorname{LSE}\left(\mathbf{z}^{c}\right) \geq \operatorname{LSE}\left(\mathbf{u}^{c}\right) - \log(C),$ where $\mathbf{u}^{c} := \left\{ f(x)^{\top} \mu_{c} \right\}_{c \in \mathcal{C}}$. Plugging this into the proof of \cref{lem:first-bound}, we obtain 
\begin{equation}
\label{eq:second-smooth-max}
    L_{\mathrm{CE}}(f, W^{\mu}) \leq \sigma + \tau L(f) +  \tau \Delta + \log(C).
\end{equation}
Finally, taking the minimum over both refined bounds \cref{eq:first-smooth-max} and \cref{eq:second-smooth-max} and then over all linear classifiers completes the proof.
\end{proof}
Combining \cref{lem:first-bound} and \cref{lem:second-bound}, as well as rewriting with the following notation \(  \beta(f, \sigma) := {\sigma/\tau} + L(f) + \Delta \) completes the proof of \cref{thm:downstream}.

\subsection{Risk Certificate for Contrastive Zero-One Risk}

In the previous sections, we demonstrated that non-vacuous risk certificates can be obtained for the SimCLR loss and that the SimCLR loss acts as a surrogate loss for downstream classification. Similarly to Nozawa et al. \cite{nozawa2020pac}, we extend our analysis to the contrastive zero-one risk. Recall the population contrastive zero-one risk $R(f)$ in \eqref{eqn:contrastive01}. Let $\widehat{R}_S(f)$ denote the empirical counterpart of $R(f)$. We extend both risk certificates from \cref{sec:risk-certitificates} to the contrastive zero-one risk.  
First, we present \cref{thm:zero-one-diarmid} corresponding to the extension of \cref{thm:simclr-diarmid-pb}.

\begin{theorem}[Extension of \cref{thm:simclr-diarmid-pb}] \label{thm:zero-one-diarmid}
If $\bar{x}_1, \ldots, \bar{x}_n \sim_{iid} \mathcal{D}_\mathcal{X}$ correspond to the unlabeled samples that generate the positive pairs in $S\sim\mathcal{S}^n$, then one can express the empirical contrastive zero-one risk as
$R_S(f) = \phi\left(\bar{x}_1, \ldots, \bar{x}_i, \ldots, \bar{x}_n\right)$,
where the map $\phi:\mathcal{X}^n\to\mathbb{R}$  satisfies the bounded difference assumption with $c_i = \frac{2}{n}$. 
As a consequence, for any confidence parameter \(\delta \in (0,1)\), with probability at least \(1-\delta\) over dataset \(S\), for all \(Q\):
\[
R(Q) \leq \widehat{R}_S(Q) + 2\sqrt{\frac{\text{KL}(Q \parallel P) + \log \frac{2n}{\delta}}{2(n-1)}}. 
\]
\end{theorem}
\begin{proof}
For the first part, we follow the proof from \cref{lem:bda-simclr}, with the only difference that $\delta_i(x_i) \leq 1$ and $\delta_j(x_i) \leq \frac{1}{m-1}$, and thus we obtain $\sup _{x_1, \ldots, x_n, x_i^{\prime} \in X} \left| \Delta \phi(x_i) \right| \leq \frac{2}{n}$. For the PAC-Bayes bound, we can follow the proof of \cref{thm:simclr-diarmid-pb} and use $C=2$.
\end{proof}

One can also extend \cref{thm:simclr-kl-pb}, as stated below. 

\begin{theorem}[Extension of \cref{thm:simclr-kl-pb}] \label{thm:kl-zero-one}
    Let \(\delta \in (0,1)\) be a confidence parameter. With probability at least \(1-\delta\) over dataset \(S\), for all \(Q\):
$$
R(Q) \leq \inf_{\alpha \in (0,1)} \left\{ \mathrm{kl}^{\text{-}1} \left( \widehat{R}_S(Q) + \gamma + \left(\frac{\delta}{2}\right)^{\frac{1-\alpha}{\alpha}}, \frac{1}{n}\left(\mathrm{KL}\left(Q \| P\right) + \log \frac{\sqrt{n}}{\delta}\right) \right)  + \gamma + \left(\frac{\delta}{2}\right)^{\frac{1}{\alpha}} \right\},
$$
where $\gamma = \sqrt{{(\log\left({2/\delta}\right))/2(m-1)\alpha}}$.
\end{theorem}

\begin{proof}
     We follow the proof from \cref{thm:simclr-kl-pb}: (1) the sum containing the negative samples becomes $S(x, X^-) := \sum_{x^{\prime} \in X^-}\mathbb{I}_{\left\{f(x)^{\top}f(x^+) < f(x)^{\top}f(x^{\prime})\right\}}$;  (2) since $\mathbb{I}_{\{\cdot\}}\leq 1$, we obtain $\varepsilon = \smash{\sqrt{((m-1)\log({2/\delta}))/2}}$ and we have $B_{\ell}=1$; (3) the upper bound on the contrastive zero-one population risk by an intermediate loss becomes  $R(Q) \leq \widetilde{R}(Q) + \frac{\varepsilon}{m-1} + \frac{\delta}{2}$ and the PAC-Bayes bound is computed using $B=1$; (4) we set $\gamma = 
 \frac{\varepsilon}{m-1} $ and take the infimum over $\alpha$.
\end{proof}

%% file: files/4_experiments.tex
\section{Experiments}\label{sec:expe}

In this section, we describe the experimental setup and empirically demonstrate that our risk certificates improved upon previous risk certificates through experiments on the CIFAR-10 dataset. For a comprehensive overview of all applicable previous bounds, please refer to the supplementary materials. The code for our experiments is available in PyTorch \cite{paszke2017automatic}. Experimental results on the MNIST dataset and additional experimental details are available in the supplementary materials.

\paragraph{Datasets and models} We use two popular benchmarks: (1) CIFAR-10, which consists of 50,000 training images and 10,000 test images \cite{krizhevsky2009learning}, and (2) MNIST, which consists of 60,000 training images and 10,000 test images \cite{lecun2010mnist}, as provided in torchvision \cite{marcel2010torchvision}. The images are preprocessed by normalizing all pixels per channel based on the training data. Data augmentation includes random cropping, resizing (with random flipping), and color distortions, as detailed in Appendix H \cite{chen2020simple}. We employ a 7-layer convolutional neural network (CNN) with max-pooling every two layers for CIFAR-10 experiments and a 3-layer CNN for MNIST experiments. We use a 2-layer MLP projection head to project to a 128-dimensional latent space, with a feature dimensionality of 2048 for CIFAR-10 and 512 for MNIST. ReLU activations are used in each hidden layer. The mean parameters $\mu_0$ of the prior are initialized randomly from a truncated centered Gaussian distribution with a standard deviation of $1/\sqrt{n_{\text{in}}}$, where $n_{\text{in}}$ is the dimension of the inputs to a particular layer, truncating at $\pm 2$ standard deviations \cite{perez2021learning}. The prior distribution scale (standard deviation $\sigma_0$) is  selected from $\{0.01, 0.05, 0.1\}$. 

\paragraph{PAC-Bayes Learning} The learning and certification strategy involves three steps: (1) choose or learn a prior from a subset of the dataset; (2) learn a posterior on the entire training dataset; (3) evaluate the risk certificate for the posterior on a subset of the dataset independent of the prior. We experiment with two types of priors: informed and random. The informed prior is learned using a subset of the training dataset via empirical risk minimization or PAC-Bayes objective minimization.  The posterior is initialized to the prior and learned using the entire training dataset by PAC-Bayes objective minimization. We use the \textit{PAC-Bayes with Backprop} (PBB) procedure \cite{perez2021tighter} and the following $f_{\text {classic}}$ objective : $f_{\text {classic }}(Q)=(1/B)\widehat{L}_S(Q)+\smash{\sqrt{(\eta \mathrm{KL}\left(Q \| P\right)+\log ({\sqrt{n}/\delta}))/2n}}$,
where $\eta$ in $[0, 1]$ is a coefficient introduced to control the influence of the KL term in the training objective, called the KL penalty. We use a KL penalty term of $10^{-6}$ for learning the prior and no penalty term for learning the posterior. We use SGD with momentum as optimizer and we perform a grid search for momentum values in $\{0.8, 0.85, 0.90, 0.95\}$ and learning rates in $\{0.1, 0.5, 1.0, 1.5\}$. Training was conducted for 100 epochs, and we selected the hyperparameters that give the best risk certificates. We experiment with different temperatures selected $\{0.2, 0.5, 0.7, 1\}$. Unless otherwise specified, experiments are run using a probabilistic prior with the simplified SimCLR loss, a batch size of $m=250$, and 80\% of the data for training the prior. 

\paragraph{Numerical Risk Certificates} Since $\widehat L_S(Q)$ is intractable, the final risk certificates are computed using Monte Carlo weight sampling.  Specifically, we approximate $\widehat L_S(Q)$ using the empirical measure \( \widehat{Q}_p = \sum_{j=1}^p \delta_{W_j} \), where \( W_1, \ldots, W_p \sim Q \) are i.i.d. samples. We compute all risk certificates with \(\delta = 0.04\), and \(p = 100\) Monte Carlo model samples. We report the risk certificates for both the contrastive loss and the contrastive zero-one risk using \cref{thm:simclr-diarmid-pb} and \cref{thm:simclr-kl-pb} for the contrastive loss, and \cref{thm:zero-one-diarmid} and \cref{thm:kl-zero-one} for the zero-one risk. To find the best value of $\alpha$ for \cref{thm:simclr-kl-pb} and \cref{thm:kl-zero-one}, we perform a grid search over $\{0.1, 0.2, 0.3, 0.4, 0.5\}$. We observe that $\alpha = 0.4$ provides the tightest risk certificates. Our risk certificates are compared with existing ones, as detailed in \cref{tab:risk_certificates_combined} and the supplementary materials. 

\paragraph{Linear Evaluation} We assess the quality of the learned representations through linear evaluation \cite{chen2020simple}: we report the cross-entropy loss and top-1 accuracy of linear classifiers trained on features either before or after the projection head. The classifiers are trained for $20$ epochs on the image classification task ($C=10$) using the Adam optimizer with a learning rate of $0.01$. Finally, we report the bounds on the downstream classification loss derived from \cref{thm:downstream} and compare it with the state of the art \cite{bao2022surrogate}. 

\begin{table}[htbp]
\centering
\begin{tabular}{c@{\hspace{0.11cm}}c@{\hspace{0.11cm}}c@{\hspace{0.11cm}}c@{\hspace{0.11cm}}c@{\hspace{0.11cm}}c@{\hspace{0.11cm}}c@{\hspace{0.31cm}}c@{\hspace{0.11cm}}c@{\hspace{0.11cm}}c@{\hspace{0.11cm}}c@{\hspace{0.11cm}}c@{\hspace{0.11cm}}c@{\hspace{0.11cm}}c@{\hspace{0.11cm}}c@{\hspace{0.11cm}}c}
\midrule
\multirow{4}{*}[-12ex]{\rotatebox{90}{Risk Certificate}} &  & \multicolumn{4}{c}{SimCLR Loss} &  & \multicolumn{4}{c}{Contrastive 0-1 Risk} \\
\cmidrule(lr){3-6} \cmidrule(lr){8-11}
 & & $\tau=1$ &  $\tau=0.7$ & $\tau=0.5$ & $\tau=0.2$ & & $\tau=1$ &  $\tau=0.7$ & $\tau=0.5$ & $\tau=0.2$ \\
\midrule
 & Test Loss & 4.945 & 4.640 & 4.257 &  2.7076 & & 0.0601 & 0.0433 & 0.0324 & 0.0199 \\
\midrule
 & kl bound (iid) & 7.164 & 7.674 & 8.246 & 9.954 & & 0.497 & 0.488 & 0.47 & 0.432 \\
 & Catoni's bound (iid) & 7.095 & 7.556 & 7.959 & 9.446 & & 0.469 & 0.466 & 0.435 & 0.417 \\
 & Classic bound (iid) & 8.475 & 8.698 & 8.910 & 10.166 & & 0.542 & 0.540 & 0.530 & 0.505 \\
 & $f$-divergence \cite{nozawa2020pac} & 27.03 & 30.138 & 33.27 & 48.472 & & 3.009 & 3.099 & 3.139 & 2.973 \\
 & Th. \ref{thm:simclr-kl-pb} (ours) & 5.537 & 5.491 & \textbf{5.492} & \textbf{6.223} & & 0.367 & 0.353 & 0.342 & 0.329 \\
 & Th. \ref{thm:simclr-diarmid-pb} (ours) & \textbf{5.203} & \textbf{5.328} & 6.269 & 43.779 & & \textbf{0.129} & \textbf{0.117} & \textbf{0.107} &  \textbf{0.093} \\
\midrule
 & $\operatorname{KL}/n$ & 0.0013 & 0.0014 & 0.0014 & 0.0013 & & -- & -- & -- & -- \\
\midrule
\end{tabular}
\caption{Comparison of risk certificates for the SimCLR loss and contrastive zero-one risk using different PAC-Bayes bounds for varying temperature values on CIFAR-10. \textit{kl bound (iid)} refers to the standard PAC-Bayes-kl bound computed over i.i.d. batches (see \cref{sec:risk-certitificates}), \textit{Catoni's bound (iid)} refers to Catoni's PAC-Bayes bound computed over i.i.d. batches \cite{nozawa2020pac}, \textit{Classic bound (iid)} refers to the classic PAC-Bayes bound computed over i.i.d. batches, and \textit{Nozawa et al.} refers to the PAC-Bayes bound based on $f$-divergence. Although Nozawa et al.'s bound uses $\chi^2$ divergence, we use $\operatorname{KL}$ divergence, which already results in vacuous bounds and would not improve with $\chi^2$, since $\operatorname{KL}(P \| Q) \leq \chi^2(P \| Q)$. We report test losses and observe that our bounds are remarkably tight. We also report the complexity term, $\operatorname{KL}/n$, where $\operatorname{KL}$ represents the Kullback-Leibler divergence between the prior and posterior distributions, and $n$ is the dataset size used to compute the risk certificate.}
\label{tab:risk_certificates_combined}
\vspace{-7mm}
\end{table}

\begin{table}[htbp]
\centering
\begin{tabular}{c@{\hspace{0.3cm}}c@{\hspace{0.3cm}}c@{\hspace{0.3cm}}c@{\hspace{0.3cm}}c@{\hspace{0.3cm}}c@{\hspace{0.3cm}}c@{\hspace{0.3cm}}c}
\midrule
 \multirow{2}{*}[-11ex]{\rotatebox{90}{Proj.}} & & $\tau=1 $ & $\tau =0.7$ & $\tau=0.5$ & $\tau=0.2$ \\
\midrule
& Bao et al. \cite{bao2022surrogate} & \textbf{3.2720} & \textbf{4.0274} & 5.3015 & 12.588 \\
& Th. \ref{thm:downstream} (ours) & \textbf{3.2720} & \textbf{4.0274} & \textbf{4.9533} &  \textbf{4.6079} \\
\midrule
& Sup. Loss &  1.903 & 1.837 & 1.7971 & 1.7547 \\
& top-1 & 0.4868 & 0.5710 & 0.6205 & 0.6677 \\
\midrule 
& Sup. Loss & 1.765 & 1.718 & 1.705 & 1.699 \\
& top-1 & 0.6350 & 0.6939 & 0.7102 & 0.7278 \\
\midrule
\end{tabular}
\caption{Comparison of upper bounds on downstream classification loss on CIFAR-10. We compare the original bound from Bao et al. with our refined bound (\cref{thm:downstream}). The supervised loss of the linear classifier trained on the projected features is reported, as it is directly related to the theoretical upper bound. Additionally, we report the supervised loss of the linear classifier trained on the features after removing the projection head. We empirically observe that the supervised loss of a linear classifier trained on the full features (without projection) is consistently lower than the loss of a linear classifier trained on the projected features, aligning with previous findings \cite{chen2020simple}. For reference, we also include top-1 accuracy.}
\label{tab:bound}

\vspace{-7mm}
\end{table}

\paragraph{Results} Table \ref{tab:risk_certificates_combined} demonstrates that our proposed risk certificates for the SimCLR loss are non-vacuous and significantly outperform existing risk certificates on CIFAR-10, closely aligning with the corresponding test losses. Interestingly, \cref{thm:simclr-kl-pb} yields tighter certificates for $\tau \leq 0.5$, while \cref{thm:simclr-diarmid-pb} is more effective for $\tau > 0.5$. We also observe that Catoni's bound is significantly tighter than the PAC-Bayes-kl bound, which is consistent with its known advantage when $\operatorname{KL}/n$ is large \cite{zhou2018non}. Unsurprisingly, the classic PAC-Bayes bound is looser than both of these bounds. Additionally, Table \ref{tab:bound} shows that \cref{thm:downstream} improves upon the bound from Bao et al., which results in exponential growth when $\tau \leq 0.5$ \cite{bao2022surrogate}. Moreover, models trained using \textit{PAC-Bayes by Backprop} achieve competitive top1 accuracy. Table \ref{tab:risk_certificates_combined} further illustrates that our risk certificates for contrastive zero-one risk are notably tight, surpassing existing certificates. Additionally, we observe that \cref{thm:kl-zero-one} consistently outperforms \cref{thm:zero-one-diarmid}. Overall, our risk certificates are competitive, even for low temperatures. PAC-Bayes learning in models with a large number of parameters remains challenging and warrants further investigation (most studies focus on 2 or 3 hidden layers). 

%% file: files/5_discussion.tex
\section{Discussion and Future Work}
\label{sec:discussion}
We have presented novel PAC-Bayesian risk certificates tailored for the SimCLR framework. Our experiments on CIFAR-10 and MNIST show that our bounds yield non-vacuous risk certificates and significantly outperform previous ones. 
\paragraph{Bounding techniques} \cref{thm:simclr-kl-pb} and \cref{thm:kl-zero-one} rely on concentration bounds to apply the PAC-Bayes bound-kl in an i.i.d. setting. Although the PAC-Bayes-kl bound was selected for its tightness, any bound that respects our proof's assumptions could be applicable. Moreover, \cref{thm:simclr-diarmid-pb} and  \cref{thm:zero-one-diarmid} were obtained using McDiarmid's inequality, integrated into  McAllester's PAC-Bayes bound. Since this variant is known to be less tight than the kl bound or Catoni's bound, it would be interesting to explore whether McDiarmid's inequality can be incorporated into these tighter bounds.
\paragraph{Extension to other (non-i.i.d.) losses} While this paper primarily focuses on contrastive learning using the SimCLR framework, the approaches we propose to address the non-i.i.d. characteristics of the SimCLR loss can be readily applied to other loss presenting similar dependence, such as ranking losses \cite{chen2009ranking}, Barlow Twins \cite{zbontar2021barlow} or VICReg \cite{CabannesKBLB23}.
While our non-i.i.d. McAllester PAC-Bayes bound (\cref{thm:simclr-diarmid-pb}) requires only a bounded difference assumption, the non-i.i.d. PAC-Bayes-kl bound (\cref{thm:simclr-kl-pb}), while more opaque, requires only a Hoeffding's assumption on the dependent terms.
\paragraph{Impact of temperature scaling and projection head} Temperature scaling in the SimCLR loss remains challenging as the scaling constant loosens the PAC-Bayes bounds, suggesting the need for more adapted PAC-Bayes bounds for this type of loss. Regarding the bound on the downstream classification loss, our approach better handles smaller temperatures than previous bounds, though it still struggles to perfectly align with downstream classification losses at low temperatures. 

Regarding projection head, we empirically observe that classification loss is lower when features are used without a projection head compared to with one, yet the theoretical role of the projection head cannot fully be understood with our downstream classification bound. In this work, we proposed an approach to integrate a simple projection head into our bound, and it would be interesting to link this with \cite{jing2021understanding} that suggests a fixed low-rank diagonal projector might suffice instead of a trainable projection head.  
\paragraph{PAC-Bayes learning} Our models trained with \textit{PAC-Bayes by Backprop} achieve accuracy competitive with \cite{chen2020simple} despite using a much smaller model (7-layer CNN vs. 50-layer CNN) and training for fewer epochs (100 vs. 500). Although this paper focuses on deriving better risk certificates rather than improving the PAC-Bayes learning algorithm, there is a need to extend the PAC-Bayes paradigm to large-scale neural networks, such as ResNet50.

\newpage

\section*{Acknowledgments}
We would like to express our gratitude to Maximilian Fleissner for providing Lemma 3.14 and for his careful reading of the original manuscript. We also sincerely thank the SIMODS reviewers for their valuable comments and constructive feedback, which greatly improved the quality of this work.

%% file: files/6_app_bounds.tex
\section{Summary of previous PAC-Bayes Bounds}
\label{app:bounds}

\subsection{Applicable to the SimCLR loss}

Below is the list of previous PAC-Bayes bounds applicable to the SimCLR loss. Recall that \(n\) denotes the size of the dataset $S$ and \(m\) the batch size. We call $U$ the dataset of $p=\frac{n}{m}$ i.i.d. batches partitioned from $S$. 

\begin{proposition}[Classic PAC-Bayes Bound over i.i.d. batches] 
    For any prior \( P \) over \( \mathcal{W} \), and any \( \delta \in (0,1) \), with a probability of at least \( 1-\delta \) over size-\( p \) i.i.d. random batches \( U \), simultaneously for all posterior distributions \( Q \) over \( \mathcal{W} \), the following inequality holds:
    \[
\frac{1}{B_\ell} L(Q) \leq \frac{1}{B_\ell} \widehat{L}_S(Q) + \sqrt{m \frac{\text{KL}(Q \parallel \mathcal{P}) + \log \frac{2\sqrt{n}}{\delta \sqrt{m}}}{2n}}
\]
\end{proposition}

\begin{proof}
We apply the classic PAC-Bayes bound over $U$  \cite{perez2021tighter}, since the bound would not hold over the dataset $S$. This implies that, instead of using the quantity $n$  in the bound, we use the quantity $p=\frac{n}{m}$ corresponding to the number of batches.
\end{proof}

\begin{proposition}[kl-PAC-Bayes Bound over i.i.d. batches]
    For any prior \( P \) over \( \mathcal{W} \), and any \( \delta \in (0,1) \), with a probability of at least \( 1-\delta \) over size-\( p \) i.i.d. random batches \( U \), simultaneously for all posterior distributions \( Q \) over \( \mathcal{W} \), the following inequality holds:
    \[
\frac{1}{B_\ell} L(Q) \leq  \operatorname{kl}^{-1}\left( \frac{1}{B_\ell} \widehat{L}_S(Q), m\frac{\text{KL}(Q \parallel \mathcal{P}) + \log \frac{2\sqrt{n}}{\delta\sqrt{m}}}{n} \right)
\]
\end{proposition}

\begin{proof}
Similarly, we apply the kl-PAC-Bayes bound over $U$  \cite{perez2021tighter}.
\end{proof}

\begin{proposition}[Catoni's PAC-Bayes Bound over i.i.d. batches]
    For any prior \( P \) over \( \mathcal{W} \), and any \( \delta \in (0,1) \), with a probability of at least \( 1-\delta \) over size-\( p \) i.i.d. random batches \( U \), simultaneously for all posterior distributions \( Q \) over \( \mathcal{W} \), the following inequality holds:
\[
\frac{1}{B_\ell} L(Q) \leq \inf_{\lambda >0} \left\{\frac{1 - \exp\left( - \frac{\lambda}{B_\ell} \widehat{L}_S(Q) - m\frac{\text{KL}(Q \parallel \mathcal{P}) + \log \frac{1}{\delta}}{n} \right)}{1 - \exp(-\lambda)} \right\}
\]
\end{proposition}

\begin{proof}
Similarly, we apply Catoni's bound over $U$  \cite{catoni2007pac}.
\end{proof}

\begin{proposition}[$f$-divergence PAC-Bayes Bound]
    For any prior \( P \) over \( \mathcal{W} \), and any \( \delta \in (0,1) \), with a probability of at least \( 1-\delta \) over size-\( p \) i.i.d. random batches \( U \), simultaneously for all posterior distributions \( Q \) over \( \mathcal{W} \), the following inequality holds:
\[
\frac{1}{B_\ell} L(Q) \leq \frac{1}{B_\ell} \widehat{L}_S(Q) + \sqrt{\frac{m-1}{n\delta} \left( \chi^2(Q \parallel P) + 1 \right)}
\]
\end{proposition}

\begin{proof}
We adapt the PAC-Bayes $f$-divergence to the SimCLR loss \cite{nozawa2020pac}. We have:
\[
 L(Q) \leq  L_S(Q) + \sqrt{\frac{\mathcal{M}_2}{\delta} \left( \chi^2(Q \parallel P) + 1 \right)},
\]
where \(\mathcal{M}_2 = \mathbb{E}_{\mathbf{f} \sim \mathcal{P}} \mathbb{E}_{S \sim \mathcal{S}^m}\left(\left|L(\mathbf{f})-\widehat{L}_S(\mathbf{f})\right|^2\right)\).  \\
\(\mathcal{M}_2\) can be upper-bounded using the following covariance:
\[
\operatorname{Cov}\left(\ell\left(\mathbf{z}_i\right), \ell\left(\mathbf{z}_j\right)\right) \begin{cases}
\leq B_{\ell}^2 & \text{if } i, j \text{ are in the same batch} \\
= 0 & \text{otherwise}
\end{cases}
\]
where \(\ell\left(\mathbf{z}_i\right) = \ell_{cont}((x_i, x_i^+), X_i^-)\). Indeed, we have \cite{alquier2018simpler}:
\[
\mathbb{E}\left[\left(\frac{1}{n} \sum_{i=1}^n \ell\left(\mathbf{z}_i\right) - \mathbb{E}_{X_i^-}\left[\ell\left(\mathbf{z}_i\right)\right]\right)^2\right] = \frac{1}{n^2} \sum_{i=1}^n \sum_{j=1}^n \operatorname{Cov}\left[\ell\left(\mathbf{z}_i\right), \ell\left(\mathbf{z}_j\right)\right]
\]
which implies for the simplified SimCLR loss:
\[
\mathcal{M}_2 \leq \frac{1}{n^2} \sum_{i=1}^n (m-1) B_{\ell}^2 = \frac{m-1}{n} B_{\ell}^2.
\]
\end{proof}

\subsection{Applicable to the contrastive zero-one risk}

\begin{proposition}[Classic PAC-Bayes Bound over i.i.d. batches]
    For any prior \( P \) over \( \mathcal{W} \), and any \( \delta \in (0,1) \), with a probability of at least \( 1-\delta \) over size-\( p \) i.i.d. random batches \( U \), simultaneously for all posterior distributions \( Q \) over \( \mathcal{W} \), the following inequality holds:
    \[
    R(Q) \leq \widehat{R}_S(Q)  + \sqrt{m \frac{\text{KL}(Q \parallel \mathcal{P}) + \log \frac{2\sqrt{n}}{\delta \sqrt{m}}}{2n}}
    \]
\end{proposition}

\begin{proposition}[kl-PAC-Bayes Bound over i.i.d. batches]
    For any prior \( P \) over \( \mathcal{W} \), and any \( \delta \in (0,1) \), with a probability of at least \( 1-\delta \) over size-\( p \) i.i.d. random batches \( U \), simultaneously for all posterior distributions \( Q \) over \( \mathcal{W} \), the following inequality holds:
    \[
    R(Q) \leq  \operatorname{kl}^{-1}\left(  \widehat{R}_S(Q), m\frac{\text{KL}(Q \parallel \mathcal{P}) + \log \frac{2\sqrt{n}}{\delta\sqrt{m}}}{n} \right)
    \]
\end{proposition}

\begin{proposition}[Catoni's PAC-Bayes Bound over i.i.d. batches]
    For any prior \( P \) over \( \mathcal{W} \), and any \( \delta \in (0,1) \), with a probability of at least \( 1-\delta \) over size-\( p \) i.i.d. random batches \( U \), simultaneously for all posterior distributions \( Q \) over \( \mathcal{W} \), the following inequality holds:
    \[
    R(Q) \leq \inf_{\lambda >0} \left\{\frac{1 - \exp\left( - \lambda  \widehat{R}_S(Q) - m\frac{\text{KL}(Q \parallel \mathcal{P}) + \log \frac{1}{\delta}}{n} \right)}{1 - \exp(-\lambda)} \right\}
    \]
\end{proposition}

\begin{proposition}[$f$-divergence PAC-Bayes Bound]
    For any prior \( P \) over \( \mathcal{W} \), and any \( \delta \in (0,1) \), with a probability of at least \( 1-\delta \) over size-\( p \) i.i.d. random batches \( U \), simultaneously for all posterior distributions \( Q \) over \( \mathcal{W} \), the following inequality holds:
    \[
    R(Q) \leq \widehat{R}_S(Q) + \sqrt{\frac{m-1}{n\delta} \left( \chi^2(Q \parallel P) + 1 \right)}
    \]
\end{proposition}

%% file: files/7_app_exp.tex
\section{Additional Experimental Details}

\subsection{Data Pre-processing Details}

For data augmentation, we apply random cropping, random horizontal flip with probability 0.5, color jittering with strength 0.5 and probability 0.8, and color dropping, leaving out gaussian blur \cite{chen2020simple, marcel2010torchvision}. We then normalize the augmented images per channel using the mean and standard deviation of the training data. 

\subsection{Computing Infrastructure }

The experiments are run using three different resource types: 
\begin{itemize}
    \item CPU + Nvidia Tesla P100 (16GB, No Tensor Cores)
    \item CPU + Nvidia Tesla V100 (16GB, Tensor Cores)
    \item CPU + Nvidia Ampere A100 (20G MIG, Ampere Tensor Cores)
\end{itemize}
We use a small memory (40GB) and the Nvidia NGC container image \textit{Pytorch 2.4.0}.

\subsection{Experiments on MNIST}

In this section, we present the results of the experiments on MNIST, as detailed in \cref{tab:exp1}, \cref{tab:exp2}, and \cref{tab:exp3}.

\begin{table}[ht]
\label{tab:exp1}
\centering
\begin{tabular}{c c c c c c c}
\specialrule{1.5pt}{0pt}{0pt}
\multirow{4}{*}[-12ex]{\rotatebox{90}{Risk Certificate}} &  & \multicolumn{4}{c}{SimCLR Loss} \\
\cmidrule(lr){3-6}
 & & $\tau=1 $ &  $\tau=0.7$ & $\tau=0.5$ & $\tau=0.2$  \\
\specialrule{1.5pt}{0pt}{0pt}
 & Test Loss & 4.9318 & 4.6575 & 4.3130 & 2.8574  \\
 \midrule
 & kl bound (iid) & 6.97 & 7.008 & 7.258 & 8.108   \\
 & Catoni's bound (iid) & 6.875 & 6.848 & 7.007 & 7.618 \\
 & Classic bound (iid) & 7.904 & 7.426 & 7.475 & 8.324 \\
 & Nozawa et al. & 23.469 & 20.448 & 22.275 & 35.05  \\
 & Th. 1 (ours) & 5.465 & 5.383 & 5.368 & 6.011  \\
 & Th. 2 (ours) & 5.099 & 5.093 & 5.729 & 36.033  \\
\midrule
 & $\operatorname{KL}/n$ & 0.0009 & 0.0005 & 0.0005 & 0.0006  \\
\specialrule{1.5pt}{0pt}{0pt}
\end{tabular}
\caption{Comparison of risk certificates for the SimCLR loss on the MNIST dataset. }
\label{tab:risk_certificates_simclr_mnist}
\end{table}

\begin{table}[ht]
\label{tab:exp2}
\centering
\begin{tabular}{c c c c c c}
\specialrule{1.5pt}{0pt}{0pt}
\multirow{4}{*}[-12ex]{\rotatebox{90}{Risk Certificate}} &  & \multicolumn{4}{c}{Contrastive 0-1 Risk} \\
\cmidrule(lr){3-6}
 & & $\tau=1 $ &  $\tau=0.7$ & $\tau=0.5$ & $\tau=0.2$ \\
\specialrule{1.5pt}{0pt}{0pt}
& Test Loss & 0.0571 & 0.0529 & 0.0478 & 0.0357 \\
 \midrule
 & kl bound (iid) & 0.408 & 0.327 & 0.324 & 0.319 \\
 & Catoni's bound (iid) & 0.419 & 0.333 & 0.331 & 0.309 \\
 & Classic bound (iid) & 0.466 & 0.396 & 0.394 & 0.399 \\
 & Nozawa et al. & 2.535 & 1.95 & 2.097 & 2.097 \\
 & Th. 4 (ours) & 0.356 & 0.347 & 0.343 & 0.333 \\
 & Th. 5 (ours) & 0.113 & 0.101 & 0.098 & 0.088 \\ 
\specialrule{1.5pt}{0pt}{0pt}
\end{tabular}
\caption{Comparison of risk certificates for the contrastive zero-one risk on the MNIST dataset. }
\label{tab:risk_certificates_risk_mnist}
\end{table}
\begin{table}[ht]
\label{tab:exp3}
\centering
\begin{tabular}{c c c c c c}
\specialrule{1.5pt}{0pt}{0pt}
 \multirow{2}{*}[-10ex]{\rotatebox{90}{Proj.}} & & $\tau=1 $ & $\tau=0.7$ & $\tau=0.5$ & $\tau=0.2$ \\
\specialrule{1.5pt}{0pt}{0pt}
& Bao et al. & 3.0212 & 3.7607 & 5.0141 & 12.5462 \\
& Th. 3 (ours) & 3.0212 & 3.7607 & 4.8096 & 4.5996 \\
\midrule
& Sup. Loss & 1.5453 & 1.5317 & 1.5063 & 1.5163 \\
& top-1 & 0.8874 & 0.9057 & 0.9425 & 0.9418 \\
\midrule 
& Sup. Loss & 1.4783 & 1.4769 & 1.4732 & 1.4793 \\
& top-1 & 0.9779 & 0.9809 & 0.9829 & 0.9778 \\
\specialrule{1.5pt}{0pt}{0pt}
\end{tabular}
\caption{Comparison of upper bounds on downstream classification loss with MNIST.}
\label{tab:bound_mnist}
\end{table}

\section{KL Divergence}

The KL divergence between one-dimensional Gaussian distributions is given by: $\operatorname{KL}\left(\operatorname{Gauss}\left(\mu_1, b_1\right) \| \operatorname{Gauss}\left(\mu_0, b_0\right)\right) = \frac{1}{2} \left(\log \left(\frac{b_0}{b_1}\right) + \frac{\left(\mu_1 - \mu_0\right)^2}{b_0} + \frac{b_1}{b_0} - 1\right)$
For multi-dimensional Gaussian distributions with diagonal covariance matrices, the KL divergence is the sum of the KL divergences of the independent components \cite{perez2021tighter}.

%% file: files/8_app_proof.tex
\section{Detailed Proof of Lemma 3.20}

Below, we provide a detailed proof of \cref{lem:second-bound}. Specifically, we aim to show that
\[
\min_{W \in \mathbb{R}^{C \times d}} L_{\mathrm{CE}}(f, W) \leq \sigma + \tau L(f) + \tau \Delta + \alpha \, ,
\]
where $\alpha = \log(C) + \min \left\{0, \log\left(C \pi^* \cosh^2(1)\right) - \tau \Delta \right\}$.
\begin{proof}
On one hand, using two applications of \eqref{eqn:smooth-max}, we obtain an upper bound on the term $\operatorname{LSE}(\mathbf{u})$ with $\mathbf{u} := \left\{f(x)^{\top}f(x^{\prime})\right\}_{x^{\prime} \in X}$:
\begin{equation}
  \operatorname{LSE}(\mathbf{u}) =  \operatorname{LSE}\left(\tau \cdot \frac{\mathbf{u}}{\tau}\right) 
  \leq \tau \max_i z_i + \log(K) 
  \leq \tau \operatorname{LSE}(\mathbf{z}) + \log(K),
\end{equation}
where the first inequality uses $t = \tau$ and the second one uses $t = 1$. This yields:
\begin{equation}
\label{eq:app-lse-lower-bound}
\operatorname{LSE}(\mathbf{z}) \geq \frac{1}{\tau} \left( \operatorname{LSE}(\mathbf{u}) - \log(K) \right).
\end{equation}
We obtain:
\begin{align*}
L(f)
&\overset{(a)}{\geq} - \frac{\sigma}{\tau} - \mathbb{E}_{(x, y)} \left[ \frac{f(x)^{\top} \mu_{y}}{\tau} \right] 
+ \mathbb{E}_x \mathbb{E}_{\substack{\{(x_i^{\prime}, y_i^{\prime})\}_{i=1}^K}} \left[ \operatorname{LSE}(\mathbf{z}) \right] \\
&\overset{(b)}{\geq} - \frac{\sigma}{\tau} - \mathbb{E}_{(x, y)} \left[ \frac{f(x)^{\top} \mu_{y}}{\tau} \right] 
+ \frac{1}{\tau} \mathbb{E}_x \mathbb{E}_{\substack{\{(x_i^{\prime}, y_i^{\prime})\}_{i=1}^K}} \left[ \operatorname{LSE}(\mathbf{u}) \right] - \frac{\log(K)}{\tau} \\
&\overset{(c)}{\geq} - \frac{\sigma}{\tau} - \mathbb{E}_{(x, y)} \left[ \frac{f(x)^{\top} \mu_{y}}{\tau} \right] 
+ \frac{1}{\tau} \mathbb{E}_x \mathbb{E}_{\substack{\{y_i^{\prime}\}_{i=1}^K}} \left[ \operatorname{LSE}(\mathbf{u}^{\mu}) \right] - \frac{\log(K)}{\tau},
\end{align*}
where: (a) follows from the sequence of inequalities in \cref{eq:lem-first-part-a} to \cref{eq:lem-first-part-d}; (b) follows by applying \cref{eq:app-lse-lower-bound}; (c) follows from an application of Jensen's inequality with $\mathbf{u}^{\mu} := \left\{ f(x)^{\top} \mu_{y^{\prime}} \right\}_{y^{\prime} \in Y}$.

Then, we derive a lower bound on $\mathbb{E}_{Y} \left[ \operatorname{LSE}(\mathbf{u}^{\mu}) \right]$:
\begin{align*}
\mathbb{E}_{Y} \left[ \operatorname{LSE}(\mathbf{u}^{\mu}) \right]
&\overset{(a)}{\geq} -\mathbb{E}_{Y} \left[ \operatorname{LSE}(-\mathbf{u}^{\mu}) \right] + 2\log(K) \\
&\overset{(b)}{\geq} -\operatorname{LSE}\left( \left\{ -f(x)^{\top} \mu_{c} \right\}_{c \in \mathcal{C}} \right) + \log(K) - \log(\pi^*) \\
&\overset{(c)}{\geq} \operatorname{LSE}\left( \left\{ f(x)^{\top} \mu_{c} \right\}_{c \in \mathcal{C}} \right) - 2\log\left(C \cosh(1)\right) + \log(K) - \log(\pi^*) \\
&\overset{(d)}{=} \operatorname{LSE}\left( \left\{ f(x)^{\top} \mu_{c} \right\}_{c \in \mathcal{C}} \right) - \log\left( \frac{C^2 \pi^*}{K} \cosh^2(1) \right),
\end{align*}
where (a) uses the left-hand side of \cref{lem:bao} for $N = K$ and $L = 1$, (b) is a direct application of \cref{eq:lem-second-part-b} to \cref{eq:lem-first-part-e}, using $\mathbf{u}^{\mu}$ instead of $\mathbf{z}^{\mu}$ (they differ only by a scaling factor $\tau$), (c) applies the right-hand side of \cref{lem:bao} for $N = C$ and $L = 1$, (d) follows by simplifying the previous line. 

Combining this with the earlier bound, we obtain:
\begin{align*}
L(f)
&\geq {-} \frac{\sigma}{\tau}{-} \mathbb{E}_{(x, y)} \left[ \frac{f(x)^{\top} \mu_{y}}{\tau} \right]
{+} \frac{1}{\tau} \operatorname{LSE}\left( \left\{ f(x)^{\top} \mu_{c} \right\}_{c \in \mathcal{C}} \right) 
{-} \frac{1}{\tau} \log\left( \frac{C^2 \pi^*}{K} \cosh^2(1) \right) 
{-} \frac{\log(K)}{\tau} \\
&= - \frac{\sigma}{\tau} + \frac{1}{\tau} L_{\mathrm{CE}}(f, W^{\mu}) 
- \frac{1}{\tau} \log\left( C^2 \pi^* \cosh^2(1) \right).
\end{align*}

Thus, we obtain the following bound on the mean classifier:
\begin{equation}
\label{eq:bound-one-improved}
L_{\mathrm{CE}}(f, W^{\mu}) \leq \sigma + \tau L(f) + \tau \Delta + \log(C) + \log\left( C \pi^* \cosh^2(1) \right) - \tau \Delta.
\end{equation}

On the other hand, applying \eqref{eqn:smooth-max} twice, we obtain an upper bound on the term $\operatorname{LSE}(\mathbf{u}^c)$, where $\mathbf{u}^{c} := \left\{ f(x)^{\top} \mu_{c} \right\}_{c \in \mathcal{C}}$:
\begin{equation}
  \operatorname{LSE}(\mathbf{u}^c) = \operatorname{LSE}\left( \tau \cdot \frac{\mathbf{u}^c}{\tau} \right) 
  \leq \tau \max_i z_i^{c} + \log C 
  \leq \tau \operatorname{LSE}(\mathbf{z}^c) + \log C,
\end{equation}
where we define $\operatorname{LSE}(\mathbf{z}^c)$ with $\mathbf{z}^{c} := \left\{ \frac{f(x)^{\top} \mu_{c}}{\tau} \right\}_{c \in \mathcal{C}}$. This yields
\begin{equation}
\label{eq:bis-lse-lower-bound}
\tau \operatorname{LSE}(\mathbf{z}^{c}) \geq \operatorname{LSE}(\mathbf{u}^{c}) - \log C.
\end{equation}
We obtain:
\begin{align*}
L(f)
&\overset{(a)}{\geq} -\frac{\sigma}{\tau} - \mathbb{E}_{(x, y)} \left[ \frac{f(x)^{\top} \mu_{y}}{\tau} \right] 
+ \mathbb{E}_{x} \mathbb{E}_{\substack{\{(x_i^{\prime}, y_i^{\prime})\}_{i=1}^K}} \left[ \operatorname{LSE}(\mathbf{z}^\mu) \right] \\
&\overset{(b)}{\geq} -\frac{\sigma}{\tau} - \mathbb{E}_{(x, y)} \left[ \frac{f(x)^{\top} \mu_{y}}{\tau} \right] 
+ \operatorname{LSE}(\mathbf{z}^c) - \Delta \\
&\overset{(c)}{\geq} -\frac{\sigma}{\tau} - \mathbb{E}_{(x, y)} \left[ \frac{f(x)^{\top} \mu_{y}}{\tau} \right] 
+ \frac{1}{\tau} \operatorname{LSE}(\mathbf{u}^c) - \frac{\log C}{\tau} - \Delta
\end{align*}
where: (a) follows from the sequence of inequalities from \cref{eq:lem-first-part-a} to \cref{eq:lem-first-part-e}; (b) follows from applying \cref{eq:lem-second-part-a} to \cref{eq:lem-second-part-g}; (c) follows by applying \cref{eq:bis-lse-lower-bound}. We obtain
\begin{equation}
\label{eq:bound-two-improved}
L_{\mathrm{CE}}(f, W^{\mu}) \leq \sigma + \tau L(f) + \tau \Delta + \log C.
\end{equation}
Finally, taking the minimum over the two refined bounds \cref{eq:bound-one-improved} and \cref{eq:bound-two-improved}, and then over all linear classifiers completes the proof.
\end{proof}

%% file: files/9_app_extensions.tex
\section{Extensions}

Throughout the paper, we have noted several directions for generalizing our results to better align with the original SimCLR framework. In this section, we elaborate on these remarks and provide explicit derivations to support each extension.

\subsection*{Remark 2.1}

We now elaborate on the statement made in \cref{rem:first}. Specifically, we detail the formulation of the original SimCLR loss introduced in \cite{chen2020simple}, applied over a dataset \( S \sim \mathcal{S}^n \):
\begin{equation}
\label{eq:original-simclr}
\widehat{L}_{S}(f) = \frac{1}{n} \sum_{i=1}^n 
\frac{\ell_{\operatorname{cont}}(x_i, x_i^+, X_i^-) + \ell_{\operatorname{cont}}(x_i^+, x_i, X_i^-)}{2} 
= \frac{1}{n} \sum_{i=1}^n \ell(x_i, x_i^+, X_i^-),
\end{equation}
where we recall that the set of negatives is given by $X_i^- = \bigcup_{j \ne i} \{x_j, x_j^+\}$.

\subsection*{Remark 3.3}

We now provide details on how the proof of \cref{lem:bda-simclr} adapts to the loss defined in \cref{eq:original-simclr}, see \cref{rem:second}. The extension requires only minor modifications:
\begin{enumerate}
    \item The set \( X_i \) is replaced with the full set of \( 2(m-1) \) negative samples, denoted \( X_i^- \).
    \item The set \( N_i \) is replaced with \( N_i^- = X_i^-\!\setminus\!\{x_i, x_i^+\} \).
    \item The quantities \( \kappa \), \( a \), and \( b \) are redefined appropriately.
\end{enumerate}
As a consequence, we obtain the following upper bound:
\[
\left| \ell_{\operatorname{cont}}(x_i, x_i^+, X_i^-) - \ell_{\operatorname{cont}}(x_i', x_i^{\prime+}, X_i^-) \right| \leq \frac{4}{\tau}.
\]
Let \( \widetilde{X}_j^- \) denote a perturbed version of the negative set \( X_j^- \). Then:
\[
\left| \ell_{\operatorname{cont}}(x_j, x_j^+, X_j^-) - \ell_{\operatorname{cont}}(x_j, x_j^+, \widetilde{X}_j^-) \right|
= \left| \log \frac{\kappa + \operatorname{sim}(x_j, x_i) + \operatorname{sim}(x_j, x_i^+)}{\kappa + \operatorname{sim}(x_j, x_i') + \operatorname{sim}(x_j, {x_i'}^+)} \right|
\leq \log \frac{\kappa + a}{\kappa + b},
\]
where \( \kappa = \operatorname{sim}(x_j, x_j^+) + \sum_{x' \in N_i^-} \operatorname{sim}(x_j, x') \), \( a = 2e^{1/\tau} \) and \( b = 2e^{-1/\tau} \). Since \( \kappa \) includes \( 2m - 3 \) terms, it is lower-bounded as $\kappa \geq (2m - 3) e^{-1/\tau}$. Thus, we can further bound:
\[
\log \frac{\kappa + a}{\kappa + b}
\leq \log \frac{(2m - 3)e^{-1/\tau} + 2e^{1/\tau}}{(2m - 3)e^{-1/\tau} + 2e^{-1/\tau}}
= \log \frac{(2m - 3) + 2e^{2/\tau}}{2m - 1}.
\]
Arguing similarly to the proof of \cref{lem:bda-simclr}  we conclude that
\[
|\delta_i(\bar{x}_i)| \leq \frac{4}{\tau} \quad \text{and} \quad |\delta_j(\bar{x}_i)| \leq \log \frac{(2m - 3) + 2e^{2/\tau}}{2m - 1}.
\]
Combining these yields $ \frac{4}{\tau} + (m - 1) \log \frac{(2m - 3) + 2e^{2/\tau}}{2m - 1} $.

\subsection*{Remark 3.9}

We now detail how the proof of \cref{thm:simclr-kl-pb} extends to the loss defined in \cref{eq:original-simclr}; see also \cref{rem:third}. In this extension, the set $X$ is replaced by $X^-$, and we therefore require a concentration bound on the quantity
\[
S(x, X^-) := \sum\limits_{x^{\prime} \in X^-} \operatorname{sim}(x, x^{\prime}) 
= \sum\limits_{x^{\prime} \in X} \left(\operatorname{sim}(x, x^{\prime}) + \operatorname{sim}(x, {x^{\prime}}^+)\right).
\]
This expression shows that $S(x, X^-)$ is a sum of $m - 1$ independent and bounded random variables, each lying within the interval \([2e^{-1/\tau}, 2e^{1/\tau}]\). We apply Hoeffding’s inequality with range parameter $c = 2\left(e^{1/\tau} - e^{-1/\tau}\right)$, and obtain the same concentration bound as in \cref{lem:one-kl}, with the only difference being that the term $\varepsilon$ is now given by $\varepsilon = 2\left(e^{1/\tau} - e^{-1/\tau}\right) \sqrt{\frac{m - 1}{2} \log \frac{1}{\delta}},$ which is multiplied by a factor of 2 compared to the previous result.

\subsection*{Remark 3.16}

As discussed in \cref{rem:fourth}, \cref{thm:downstream} can be extended to include a simple projection head. Specifically, we first apply the theorem to the projected features \( f_1 \), which yields
\begin{equation*}
    \min_{W^{(1)} \in \mathbb{R}^{C \times k}} L_{\mathrm{CE}}(f_1, W^{(1)}) \leq \min 
    \left\{
        \beta(f_1, \sigma),~ 
        \tau \beta(f_1, \sigma) + \alpha
    \right\},
\end{equation*}
where \( \alpha \), \( \sigma \), and \( \beta \) are defined in \cref{thm:downstream}. Now, since \( L_{\mathrm{CE}}(f_1, W^{(1)}) = L_{\mathrm{CE}}(f, \tilde{W}) \) for 
\[
\tilde{W} = 
\begin{bmatrix}
    W^{(1)} \\
    0_{(d-k) \times C}
\end{bmatrix},
\]
the bound also holds for any \( W \in \mathbb{R}^{C \times d} \) such that \( L_{\mathrm{CE}}(f, W) \leq L_{\mathrm{CE}}(f, \tilde{W}) \). In particular, this includes the minimizer over all such \( W \). Therefore, we conclude:
\begin{equation*}
    \min_{W \in \mathbb{R}^{C \times d}} L_{\mathrm{CE}}(f, W) \leq \min 
    \left\{
        \beta(f_1, \sigma),~ 
        \tau \beta(f_1, \sigma) + \alpha
    \right\}.
\end{equation*}

\subsection*{Remark 3.17}

As mentioned in \cref{rem:fifth}, the proof of \cref{thm:downstream}  can be directly extended to the loss defined in \cref{eq:original-simclr} by replacing $X$ with $X^-$ and setting $K = 2(m - 1)$ throughout the derivation. Note that the proof remains unchanged, as it does not rely on any independence assumptions regarding the negative samples. 

\newpage
